\documentclass[11pt]{article}

\usepackage{amsmath, amsfonts, amssymb, amsthm,  graphicx, mathtools, enumerate}
\usepackage{multirow, float, algorithm, algpseudocode, changepage}

\usepackage[blocks, affil-it]{authblk}

\usepackage[numbers, square]{natbib}
\usepackage[CJKbookmarks=true,
            bookmarksnumbered=true,
			bookmarksopen=true,
			backref=section,
			colorlinks=true,
			citecolor=red,
			linkcolor=blue,
			anchorcolor=red,
			urlcolor=blue]{hyperref}
\usepackage[usenames]{color}

\usepackage[letterpaper, left=1.2truein, right=1.2truein, top = 1.2truein, bottom = 1.2truein]{geometry}

\usepackage{prettyref,soul}

\newtheorem{lemma}{Lemma}[section]
\newtheorem{proposition}{Proposition}[section]
\newtheorem{thm}{Theorem}[section]

\newrefformat{eq}{(\ref{#1})}
\newrefformat{chap}{Chapter~\ref{#1}}
\newrefformat{sec}{Section~\ref{#1}}
\newrefformat{algo}{Algorithm~\ref{#1}}
\newrefformat{fig}{Fig.~\ref{#1}}
\newrefformat{tab}{Table~\ref{#1}}
\newrefformat{rmk}{Remark~\ref{#1}}
\newrefformat{clm}{Claim~\ref{#1}}
\newrefformat{def}{Definition~\ref{#1}}
\newrefformat{cor}{Corollary~\ref{#1}}
\newrefformat{lmm}{Lemma~\ref{#1}}
\newrefformat{lemma}{Lemma~\ref{#1}}
\newrefformat{prop}{Proposition~\ref{#1}}
\newrefformat{app}{Appendix~\ref{#1}}
\newrefformat{ex}{Example~\ref{#1}}
\newrefformat{exer}{Exercise~\ref{#1}}
\newrefformat{soln}{Solution~\ref{#1}}
\newrefformat{cond}{Condition~\ref{#1}}

\def\text#1{\mbox{\rm #1}}

\newcommand{\argsup}{\mathop{\rm argsup}}
\newcommand{\arginf}{\mathop{\rm arginf}}

\newcommand{\norm}[1]{\|{#1} \|}

\newcommand{\wh}{\widehat}
\newcommand{\wt}{\widetilde}

\newcommand{\opnorm}[1]{\|#1\|_{\rm op}}

\newcommand{\JS}{{\sf JS}}
\newcommand{\sig}{{\sf sigmoid}}
\newcommand{\relu}{{\sf ReLU}}

\title{Robust Estimation and Generative Adversarial Nets\thanks{A short version of the manuscript is published in ICLR 2019.}
}
\author{Chao Gao$^1$, Jiyi Liu$^2$, Yuan Yao$^3$ and Weizhi Zhu$^3$\\
~\\
$^1$University of Chicago, $^2$Yale University \\
and $^3$Hong Kong University of Science and Technology
}

\begin{document}
\maketitle

\begin{abstract}
Robust estimation under Huber's $\epsilon$-contamination model has become an important topic in statistics and theoretical computer science. Statistically optimal procedures such as Tukey's median and other estimators based on depth functions are impractical because of their computational intractability. In this paper, we establish an intriguing connection between $f$-GANs and various depth functions through the lens of $f$-Learning. Similar to the derivation of $f$-GANs, we show that these depth functions that lead to statistically optimal robust estimators can all be viewed as variational lower bounds of the total variation distance in the framework of $f$-Learning. This connection opens the door of computing robust estimators using tools developed for training GANs. In particular, we show in both theory and experiments that some appropriate structures of discriminator networks with hidden layers in GANs lead to statistically optimal robust location estimators for both Gaussian distribution and general elliptical distributions where first moment may not exist.
\smallskip

\end{abstract}


\section{Introduction}

In the setting of Huber's $\epsilon$-contamination model \citep{huber1964robust,huber1965robust}, one has i.i.d observations
\begin{equation}
X_1,...,X_n\sim (1-\epsilon)P_{\theta}+\epsilon Q, \label{eq:Huber}
\end{equation}
and the goal is to estimate the model parameter $\theta$. Under the data generating process (\ref{eq:Huber}), each observation has a $1-\epsilon$ probability to be drawn from $P_{\theta}$ and the other $\epsilon$ probability to be drawn from the contamination distribution $Q$. The presence of an unknown contamination distribution poses both statistical and computational challenges. For example, consider a normal mean estimation problem with $P_{\theta}=N(\theta,I_p)$. Due to the contamination of data, the sample average, which is optimal when $\epsilon=0$, can be arbitrarily far away from the true mean if $Q$ charges a positive probability at infinity. Moreover, even robust estimators such as coordinatewise median and geometric median are proved to be suboptimal under the setting of (\ref{eq:Huber}) \citep{chen2018robust,diakonikolas2016robust,lai2016agnostic}. The search for both statistically optimal and computationally feasible procedures has become a fundamental problem in areas including robust statistics and theoretical computer science.

It has been shown in \cite{chen2016general} that the minimax rate $\mathcal{R}(\epsilon)$ of estimating $\theta$ under Huber's $\epsilon$-contamination model takes the form of $\mathcal{R}(\epsilon)\asymp\mathcal{R}(0)\vee\omega(\epsilon,\Theta)$, where $\mathcal{R}(0)$ is the minimax rate of the problem when $\epsilon=0$, and $\omega(\epsilon,\Theta)$ is the modulus of continuity \citep{donoho1991geometrizing} between the loss function and the total variation distance with respect to the parameter space $\Theta$. The two terms in the minimax rate characterize the difficulty of the problem with both the statistical complexity and the influence of contamination. For the normal mean estimation problem, the minimax rate with respect to the squared $\ell_2$ loss is $\frac{p}{n}\vee\epsilon^2$ \citep{chen2018robust}, and is achieved by Tukey's median \citep{tukey1975mathematics}, defined as
\begin{equation}
\wh{\theta} = \argsup_{\eta\in\mathbb{R}^p}\inf_{\|u\|=1}\frac{1}{n}\sum_{i=1}^n\mathbb{I}\{u^T(X_i-\eta)\geq 0\},\label{eq:Tukey}
\end{equation}
the maximizer of Tukey's halfspace depth.
Despite the statistical optimality of Tukey's median, computation of (\ref{eq:Tukey}) is not tractable. In fact, even an approximate algorithm takes $O(e^{Cp})$ in time \citep{amenta2000regression,chan2004optimal,rousseeuw1998computing}.

Recent developments in theoretical computer science are focused on the search of computationally tractable algorithms for estimating $\theta$ under Huber's $\epsilon$-contamination model (\ref{eq:Huber}). The success of the efforts started from two fundamental papers \cite{diakonikolas2016robust,lai2016agnostic}, where two different but related computational strategies ``iterative filtering" and ``dimension halving" were proposed to robustly estimate the normal mean. These algorithms can provably achieve the minimax rate $\frac{p}{n}\vee\epsilon^2$ up to a poly-logarithmic factor in polynomial time. The main idea behind the two methods is the fact that a good robust moment estimator can be certified efficiently by higher moments. This idea was later further extended  \citep{diakonikolas2017being,du2017computationally,diakonikolas2016bayesian,diakonikolas2018sever,diakonikolas2018efficient,diakonikolas2018list,kothari2018robust} to develop robust and computable procedures for various other problems.


Compared with these computationally feasible procedures proposed in the recent literature for robust estimation, Tukey's median (\ref{eq:Tukey}) and other depth-based estimators \citep{rousseeuw1999regression,mizera2002depth,zhang2002some,mizera2004location,paindaveine2017halfspace} have some indispensable advantages in terms of their statistical properties. First, the depth-based estimators have clear objective functions that can be interpreted from the perspective of projection pursuit \citep{mizera2002depth}. Second, the depth-based procedures are adaptive to nuisance parameters in the models such as covariance structures, contamination proportion, and error distributions \citep{chen2018robust,gao2017robust}. In comparison, many of the computationally feasible procedures for robust mean estimation in the literature rely on the knowledge of covariance matrix, and sometimes the order of the contamination proportion as well. Even though these assumptions can be relaxed, nontrivial modifications of the algorithms are required for such extensions and sometimes statistical error rates will be affected. Last but not least, Tukey's depth and other depth functions are mostly designed for robust quantile estimation, while the recent advancements in the theoretical computer science literature are all focused on robust moments estimation. Although this is not an issue when it comes to the problem of normal mean estimation, the difference becomes fundamental for robust location estimation under general settings such as elliptical distributions where moments do not necessarily exist. For a thorough overview of statistical properties of depth-based estimators, we refer the readers to \cite{liu1999multivariate,zuo2000general,zuo2018general}.

Given the desirable statistical properties discussed above, this paper is focused on the development of computational strategies of depth-like procedures. Our key observation is that robust estimators that are maximizers of depth functions, including halfspace depth, regression depth and covariance matrix depth, can all be derived under the framework of $f$-GAN \citep{nowozin2016f}. As a result, these depth-based estimators can be viewed as minimizers of variational lower bounds of the total variation distance between the empirical measure and the model distribution. This observation allows us to leverage the recent developments in the deep learning literature to compute these variational lower bounds through neural network approximations. Our theoretical results give insights on how to choose appropriate neural network classes that lead to minimax optimal robust estimation under Huber's $\epsilon$-contamination model. The main contributions of the paper are listed below.
\begin{enumerate}
\item We identify an important subclass of $f$-GAN, called $f$-Learning (Section \ref{sec:f-learn}), which helps us to unify the understandings of various depth-based estimators, GANs, and MLE in a single framework. The connection between depth functions and $f$-GAN allows us to develop depth-like estimators that not only share good statistical properties of (\ref{eq:Tukey}), but can also be trained by stochastic gradient ascent/descent algorithms.
\item In order to choose an appropriate discriminator class for robust estimation, we establish a relation between (JS)-GAN optimization and feature matching (Proposition \ref{prop:local-JS}). This implies the necessity of hidden layers of neural network structures used in the GAN training. A neural network class without hidden layer is equivalent to matching linear features, and is thus not suitable for robust estimation.
\item We prove that rate-optimal robust location estimation for both Gaussian distribution (Theorem \ref{thm:TV-NN1} for TV-GAN and Theorem \ref{thm:JS-NN2} for JS-GAN with bounded activations, and Theorem \ref{thm:JS-DNN} for deep ReLU networks) and the general family of elliptical distributions (Theorem \ref{thm:elliptical}) can be achieved by GANs that use neural network discriminator classes with appropriate structures and regularizations. Extensive numerical experiments are conducted to verify our theoretical findings and show that these procedures can be computed in practice.
\end{enumerate}

Our work is also related to the recent literature on the investigation of statistical properties of GAN. For example, nonparametric density estimation using GAN is studied by \cite{liang2017well}. Provable guarantees of learning Gaussian distributions with quadratic discriminators are established by \cite{feizi2017understanding}. Theoretical guarantees of learning Gaussian mixtures, exponential families and injective neural network generators are obtained by \cite{bai2018approximability}. The result we obtain in this paper is the first theoretical guarantee of GAN in robust estimation under Huber's $\epsilon$-contamination model.

The rest of the paper is organized as follows. In Section \ref{sec:framework}, we introduce an $f$-Learning framework and discuss the connection between robust estimation and $f$-GAN. The theoretical results of robust Gaussian mean estimation using $f$-GAN are given in Section \ref{sec:main}. Results for deep ReLU networks are given in Section \ref{sec:deep}. An extension to robust location estimation for the family of Elliptical distributions is presented in Section \ref{sec:ellip} that includes both Gaussian distribution and Cauchy distribution whose moments do not exist. In Section \ref{sec:num}, we present extensive numerical studies of the proposed procedures. Section \ref{sec:disc} collects some discussions on the results of the paper and several possible extensions of the work. Finally, all the technical proofs are given in Section \ref{sec:pf}.

We close this section by introducing the notations used in the paper. For $a,b\in\mathbb{R}$, let $a\vee b=\max(a,b)$ and $a\wedge b=\min(a,b)$. For an integer $m$, $[m]$ denotes the set $\{1,2,...,m\}$. Given a set $S$, $|S|$ denotes its cardinality, and $\mathbb{I}_S$ is the associated indicator function. For two positive sequences $\{a_n\}$ and $\{b_n\}$, the relation $a_n\lesssim b_n$ means that $a_n\leq Cb_n$ for some constant $C>0$, and $a_n\asymp b_n$ if both $a_n\lesssim b_n$ and $b_n\lesssim a_n$ hold. For a vector $v\in\mathbb{R}^p$, $\norm{v}$ denotes the $\ell_2$ norm, $\|v\|_{\infty}$ the $\ell_{\infty}$ norm, annd $\|v\|_1$ the $\ell_1$ norm. For a matrix $A\in\mathbb{R}^{d_1\times d_2}$, we use $\opnorm{A}$ to denote its operator norm, which is its largest singular value. We use $\mathbb{P}$ and $\mathbb{E}$ to denote generic probability and expectation whose distribution is determined from the context. The symbol $E_P$ is used for the expectation operator under the distribution $P$. The sigmoid function and the rectified linear unit function (ReLU) are denoted by $\sig(x)=\frac{1}{1+e^{-x}}$ and $\relu(x)=\max(x,0)$.

\section{Robust Estimation and $f$-GAN}\label{sec:framework}

We start with the definition of $f$-divergence \citep{csiszar1964informationstheoretische,ali1966general}. Given a strictly convex function $f$ that satisfies $f(1)=0$, the $f$-divergence between two probability distributions $P$ and $Q$ is defined by
\begin{equation}
D_f(P\|Q)=\int f\left(\frac{p}{q}\right)dQ.\label{eq:f-div}
\end{equation}
Here, we use $p(\cdot)$ and $q(\cdot)$ to stand for the density functions of $P$ and $Q$ with respect to some common dominating measure. For a fully rigorous definition, see \cite{polyanskiy2014lecture}. Let $f^*$ be the convex conjugate of $f$. That is, $f^*(t)=\sup_{u\in\text{dom}_f}(ut-f(u))$. A variational lower bound of (\ref{eq:f-div}) is
\begin{equation}
D_f(P\|Q) \geq \sup_{T\in\mathcal{T}}\left[E_PT(X)-E_Qf^*(T(X))\right]. \label{eq:f-div-variational}
\end{equation}
Note that the inequality (\ref{eq:f-div-variational}) becomes an equality whenever the class $\mathcal{T}$ contains the function $f'\left(p/q\right)$ \citep{nguyen2010estimating}. For notational simplicity, we also use $f'$ for an arbitrary element of the subdifferential when the derivative does not exist. With i.i.d. observations $X_1,...,X_n\sim P$, the variational lower bound (\ref{eq:f-div-variational}) naturally leads to the following learning method
\begin{equation}
\wh{P}=\arginf_{Q\in\mathcal{Q}}\sup_{T\in\mathcal{T}}\left[\frac{1}{n}\sum_{i=1}^nT(X_i)-E_Qf^*(T(X))\right]. \label{eq:f-GAN}
\end{equation}
The formula (\ref{eq:f-GAN}) is a powerful and general way to learn the distribution $P$ from its i.i.d. observations. It is known as $f$-GAN \citep{nowozin2016f}, an extension of GAN \citep{goodfellow2014generative}, which stands for \emph{generative adversarial nets}. The idea is to find a $\wh{P}$ so that the best discriminator $T$ in the class $\mathcal{T}$ cannot tell the difference between $\wh{P}$ and the empirical distribution $\frac{1}{n}\sum_{i=1}^n\delta_{X_i}$.

\subsection{$f$-Learning: A Unified Framework} \label{sec:f-learn}

Our $f$-Learning framework is based on a special case of the variational lower bound (\ref{eq:f-div-variational}). That is,
\begin{equation}
D_f(P\|Q)\geq \sup_{\wt{Q}\in\wt{\mathcal{Q}}_{Q}}\left[E_Pf'\left(\frac{\wt{q}(X)}{q(X)}\right)-E_Qf^*\left(f'\left(\frac{\wt{q}(X)}{q(X)}\right)\right)\right],\label{eq:f-Learning-pop}
\end{equation}
where $\wt{q}(\cdot)$ stands for the density function of $\wt{Q}$. Note that here we allow the class $\wt{\mathcal{Q}}_{Q}$ to depend on the distribution $Q$ in the second argument of $D_f(P\|Q)$. Compare (\ref{eq:f-Learning-pop}) with (\ref{eq:f-div-variational}), and it is easy to realize that (\ref{eq:f-Learning-pop}) is a special case of (\ref{eq:f-div-variational}) with
\begin{equation}
\mathcal{T} = \mathcal{T}_Q = \left\{f'\left(\frac{\wt{q}}{q}\right): \wt{q}\in\wt{\mathcal{Q}}_Q\right\}.\label{eq:class-T-LR}
\end{equation}
Moreover, the inequality (\ref{eq:f-Learning-pop}) becomes an equality as long as $P\in\wt{\mathcal{Q}}_Q$. The sample version of (\ref{eq:f-Learning-pop}) leads to the following learning method
\begin{equation}
\wh{P} = \arginf_{Q\in\mathcal{Q}}\sup_{\wt{Q}\in\wt{\mathcal{Q}}_Q}\left[\frac{1}{n}\sum_{i=1}^nf'\left(\frac{\wt{q}(X_i)}{q(X_i)}\right)-E_Qf^*\left(f'\left(\frac{\wt{q}(X)}{q(X)}\right)\right)\right]. \label{eq:f-Learning}
\end{equation}
The learning method (\ref{eq:f-Learning}) will be referred to as \emph{$f$-Learning} in the sequel. It is a very general framework that covers many important learning procedures as special cases. For example, consider the special case where $\wt{\mathcal{Q}}_Q=\wt{\mathcal{Q}}$ independent of $Q$, $\mathcal{Q}=\wt{\mathcal{Q}}$, and $f(x)=x\log x$. Direct calculations give $f'(x)=\log x+1$ and $f^*(t)=e^{t-1}$. Therefore, (\ref{eq:f-Learning}) becomes
$$\wh{P}=\arginf_{Q\in\mathcal{Q}}\sup_{\wt{Q}\in\mathcal{Q}}\frac{1}{n}\sum_{i=1}^n\log\frac{\wt{q}(X_i)}{q(X_i)}=\argsup_{Q\in\mathcal{Q}}\frac{1}{n}\sum_{i=1}^n\log q(X_i),$$
which is the \emph{maximum likelihood estimator (MLE)}.

The $f$-Learning (\ref{eq:f-Learning}) is related to but is different from the rho-estimation framework \citep{baraud2016rho,baraud2017new}. The unpenalized version of the rho-estimator is defined by
$$\wh{P}=\arginf_{Q\in\mathcal{Q}}\sup_{\wt{Q}\in\mathcal{Q}}\frac{1}{n}\sum_{i=1}^n\psi\left(\sqrt{\frac{\wt{q}(X_i)}{q(X_i)}}\right),$$
where $\psi:[0,+\infty]\rightarrow [-1,1]$ is a non-decreasing function that satisfies $\psi(x)=-\psi(1/x)$. The rho-estimation framework has a different motivation. The function $\psi$ is designed to generalize the logarithmic function (which leads to the MLE) so that the induced procedure is robust to a Hellinger model misspecification. On the other hand, the $f$-Learning (\ref{eq:f-Learning}) is directly derived from a variational lower bound of the $f$-divergence.

\subsection{TV-Learning and Depth-Based Estimators} \label{sec:f-learndepth}

An important generator $f$ that we will discuss here is $f(x)=(x-1)_+$. This leads to the total variation distance $D_f(P\|Q)=\frac{1}{2}\int|p-q|$. With $f'(x)=\mathbb{I}\{x\geq 1\}$ and $f^*(t)=t\mathbb{I}\{0\leq t\leq 1\}$, the \emph{TV-Learning} is given by
\begin{equation}
\wh{P}=\arginf_{Q\in\mathcal{Q}}\sup_{\wt{Q}\in\mathcal{Q}_Q}\left[\frac{1}{n}\sum_{i=1}^n\mathbb{I}\left\{\frac{\wt{q}(X_i)}{q(X_i)}\geq 1\right\}-Q\left(\frac{\wt{q}}{q}\geq 1\right)\right]. \label{eq:TV-Learning}
\end{equation}
The TV-Learning (\ref{eq:TV-Learning}) is a very useful tool in robust estimation. A closely related idea was previously explored by \cite{yatracos1985rates,devroye2012combinatorial}. We illustrate its applications by several examples of depth-based estimators. 

In the first example, consider
$$\mathcal{Q}=\left\{N(\eta,I_p):\eta\in\mathbb{R}^p\right\},\quad \wt{\mathcal{Q}}_{\eta}=\left\{N(\wt{\eta},I_p):\|\wt{\eta}-\eta\|\leq r\right\}.$$
In other words, $\mathcal{Q}$ is the class of Gaussian location family, and $\wt{\mathcal{Q}}_{\eta}$ is taken to be a subset in a local neighborhood of $N(\eta,I_p)$. Then, with $Q=N(\eta,I_p)$ and $\wt{Q}=N(\wt{\eta},I_p)$, the event $\wt{q}(X)/q(X)\geq 1$ is equivalent to $\|X-\wt{\eta}\|^2\leq\|X-\eta\|^2$. Since $\|\wt{\eta}-\eta\|\leq r$, we can write $\wt{\eta}=\eta+\wt{r}u$ for some $\wt{r}\in\mathbb{R}$ and $u\in\mathbb{R}^p$ that satisfy $0\leq \wt{r}\leq r$ and $\|u\|=1$. Then, (\ref{eq:TV-Learning}) becomes
\begin{equation}
\wh{\theta}=\arginf_{\eta\in\mathbb{R}^p}\sup_{\substack{\|u\|=1\\0\leq \wt{r}\leq r}}\left[\frac{1}{n}\sum_{i=1}^n\mathbb{I}\left\{u^T(X_i-\eta)\geq\frac{\wt{r}}{2}\right\}-\mathbb{P}\left(N(0,1)\geq\frac{\wt{r}}{2}\right)\right].\label{eq:TV-Learning-finite-r}
\end{equation}
Letting $r\rightarrow 0$, we obtain (\ref{eq:Tukey}), the exact formula of \emph{Tukey's median}.
A traditional understanding of Tukey's median is that (\ref{eq:Tukey}) maximizes the halfspace depth \citep{donoho1992breakdown} so that $\wh{\theta}$ is close to the center of all one-dimensional projections of the data. In the $f$-Learning framework, $N(\wh{\theta},I_p)$ is understood to be the minimizer of a variational lower bound of the total variation distance.

The next example is a linear model $y|X\sim N(X^T\theta,1)$. Consider the following classes
\begin{eqnarray*}
\mathcal{Q} &=& \left\{P_{y,X}=P_{y|X}P_X: P_{y|X}=N(X^T\eta,1), \eta\in\mathbb{R}^p\right\}, \\
\wt{\mathcal{Q}}_{\eta} &=& \left\{P_{y,X}=P_{y|X}P_X: P_{y|X}=N(X^T\wt{\eta},1), \|\wt{\eta}-\eta\|\leq r\right\}.
\end{eqnarray*}
Here, $P_{y,X}$ stands for the joint distribution of $y$ and $X$. The two classes $\mathcal{Q}$ and $\wt{\mathcal{Q}}_{\eta}$ share the same marginal distribution $P_X$ and the conditional distributions are specified by $N(X^T\eta,1)$ and $N(X^T\wt{\eta},1)$, respectively. Follow the same derivation of Tukey's median, let $r\rightarrow 0$, and we obtain
\begin{equation}
\wh{\theta} = \argsup_{\eta\in\mathbb{R}^p}\inf_{\|u\|=1}\frac{1}{n}\sum_{i=1}^n\mathbb{I}\{u^TX_i(y_i-X_i^T\eta)\geq 0\}, \label{eq:regression-depth}
\end{equation}
which is the estimator that maximizes the \emph{regression depth} proposed by \cite{rousseeuw1999regression}. It is worth noting that the derivation of (\ref{eq:regression-depth}) does not depend on the marginal distribution $P_X$.

The last example is on covariance matrix estimation. For this task, we set $\mathcal{Q}=\{N(0,\Gamma):\Gamma\in\mathcal{E}_p\}$, where $\mathcal{E}_p$ is the class of all $p\times p$ covariance matrices. Inspired by the derivations of Tukey depth and regression depth, it is tempting to choose $\wt{\mathcal{Q}}_{\Gamma}$ in the neighborhood of $N(0,\Gamma)$. However, a naive choice would lead to a definition that is not even Fisher consistent. We propose a rank-one neighborhood, given by
\begin{equation}
\wt{\mathcal{Q}}_{\Gamma}=\left\{N(0,\wt{\Gamma}): \wt{\Gamma}^{-1}=\Gamma^{-1}+\wt{r}uu^T\in\mathcal{E}_p, |\wt{r}|\leq r, \|u\|=1\right\}.\label{eq:inv-nb}
\end{equation}
Then, a direct calculation gives
\begin{equation}
\mathbb{I}\left\{\frac{dN(0,\wt{\Gamma})}{dN(0,\Gamma)}(X)\geq 1\right\}=\mathbb{I}\left\{ \wt{r}|u^TX|^2\leq\log(1+\wt{r}u^T\Gamma u)\right\}.\label{eq:event-cov}
\end{equation}
Since $\lim_{\wt{r}\rightarrow 0} \frac{\log(1+\wt{r}u^T\Gamma u)}{\wt{r}u^T\Gamma u}=1$, the limiting event of (\ref{eq:event-cov}) is either $\mathbb{I}\{|u^TX|^2\leq u^T\Gamma u\}$ or $\mathbb{I}\{|u^TX|^2\geq u^T\Gamma u\}$, depending on whether $\wt{r}$ tends to zero from left or from right. Therefore, with the above $\mathcal{Q}$ and $\wt{\mathcal{Q}}_{\Gamma}$, (\ref{eq:TV-Learning}) becomes
\begin{eqnarray}
\label{eq:matrix-depth2} \wh{\Sigma} &=& \arginf_{\Gamma\in\mathcal{E}_p}\sup_{\|u\|=1}\Bigg[\left(\frac{1}{n}\sum_{i=1}^n\mathbb{I}\{|u^TX_i|^2\leq u^T\Gamma u\}-\mathbb{P}(\chi_1^2\leq 1)\right) \\
\nonumber && \qquad\qquad\qquad \vee\left(\frac{1}{n}\sum_{i=1}^n\mathbb{I}\{|u^TX_i|^2> u^T\Gamma u\}-\mathbb{P}(\chi_1^2> 1)\right)\Bigg],
\end{eqnarray}
under the limit $r\rightarrow 0$. Even though the definition of (\ref{eq:inv-nb}) is given by a rank-one neighborhood of the inverse covariance matrix, the formula (\ref{eq:matrix-depth2}) can also be derived with $\wt{\Gamma}^{-1}=\Gamma^{-1}+\wt{r}uu^T$ in (\ref{eq:inv-nb}) replaced by $\wt{\Gamma}=\Gamma+\wt{r}uu^T$ by applying the Sherman-Morrison formula.
A similar formula to (\ref{eq:matrix-depth2}) in the literature is given by
\begin{equation}
\wh{\Sigma} = \argsup_{\Gamma\in\mathcal{E}_p}\inf_{\|u\|=1}\left[\frac{1}{n}\sum_{i=1}^n\mathbb{I}\{|u^TX_i|^2\leq \beta u^T\Gamma u\}\wedge \frac{1}{n}\sum_{i=1}^n\mathbb{I}\{|u^TX_i|^2\geq \beta u^T\Gamma u\}\right], \label{eq:matrix-depth}
\end{equation}
which is recognized as the maximizer of what is known as the \emph{covariance matrix depth} function \citep{zhang2002some,chen2018robust,paindaveine2017halfspace}. The $\beta$ in (\ref{eq:matrix-depth}) is a scalar defined through the equation $\mathbb{P}(N(0,1)\leq \sqrt{\beta})=3/4$. It is proved in \cite{chen2018robust} that $\wh{\Sigma}$ achieves the minimax rate under Huber's $\epsilon$-contamination model. While the formula (\ref{eq:matrix-depth2}) can be derived from TV-Learning with discriminators in the form of $\mathbb{I}\left\{\frac{dN(0,\wt{\Gamma})}{dN(0,\Gamma)}(X)\geq 1\right\}$, a special case of (\ref{eq:class-T-LR}), the formula (\ref{eq:matrix-depth}) can be derived directly from TV-GAN with discriminators in the form of $\mathbb{I}\left\{\frac{dN(0,\beta\wt{\Gamma})}{dN(0,\beta\Gamma)}(X)\geq 1\right\}$ by following a similar rank-one neighborhood argument.

\subsection{From $f$-Learning to $f$-GAN}

The depth-based estimators (\ref{eq:Tukey}), (\ref{eq:regression-depth}) and (\ref{eq:matrix-depth}) are all proved to be statistically optimal under Huber's contamination model \citep{chen2018robust,gao2017robust}. This shows the importance of TV-Learning in robust estimation. However, it is well-known that depth-based estimators are very hard to compute \citep{amenta2000regression,van1999efficient,rousseeuw1998computing}, which limits their applications only for very low-dimensional problems. On the other hand, the general $f$-GAN framework (\ref{eq:f-GAN}) has been successfully applied to learn complex distributions and images in practice \citep{goodfellow2014generative,radford2015unsupervised,salimans2016improved}. The major difference that gives the computational advantage to $f$-GAN is its flexibility in terms of designing the discriminator class $\mathcal{T}$ using neural networks compared with the pre-specified choice (\ref{eq:class-T-LR}) in $f$-Learning. While $f$-Learning provides a unified perspective in understanding various depth-based procedures in robust estimation, we can step back into the more general $f$-GAN for its computational advantages, and to design efficient computational strategies. 
However, there are at least two questions that are unclear:
\begin{enumerate}
\item How to choose the function $f$ that leads to robust learning procedures which are easy to optimize?
\item How to specify the discriminator class to learn the parameter of interest with minimax rate under Huber's $\epsilon$-contamination model?
\end{enumerate}
In the rest of the paper, we will study a robust mean estimation problem in detail to answer these questions and illustrate the power of $f$-GAN in robust estimation.

\section{Robust Mean Estimation via GAN}\label{sec:main}

In this section, we focus on the problem of robust mean estimation under Huber's $\epsilon$-contamination model. Our goal is to reveal how the choice of the class of discriminators affects robustness and statistical optimality under the simplest possible setting. That is, we have i.i.d. observations $X_1,...,X_n\sim (1-\epsilon)N(\theta,I_p)+\epsilon Q$, and we need to estimate the unknown location $\theta\in\mathbb{R}^p$ with the contaminated data. Our goal is to achieve the minimax rate $\frac{p}{n}\vee\epsilon^2$ with respect to the squared $\ell_2$ loss uniformly over all $\theta\in\mathbb{R}^p$ and all $Q$.

\subsection{Results for TV-GAN}\label{sec:TV-GAN}

We start with the total variation GAN (TV-GAN) with $f(x)=(x-1)_+$ in (\ref{eq:f-GAN}). For the Gaussian location family, (\ref{eq:f-GAN}) can be written as
\begin{equation}
\wh{\theta}=\arginf_{\eta\in\mathbb{R}^p}\sup_{D\in\mathcal{D}}\left[\frac{1}{n}\sum_{i=1}^nD(X_i)-E_{N(\eta,I_p)}D(X)\right],\label{eq:TV-GAN-NN1}
\end{equation}
with $T(x)=D(x)$ in (\ref{eq:f-GAN}).
Now we need to specify the class of discriminators $\mathcal{D}$ to solve the classification problem between $N(\eta,I_p)$ and the empirical distribution $\frac{1}{n}\sum_{i=1}^n\delta_{X_i}$. One of the simplest discriminator classes is the logistic regression,
\begin{equation}
\mathcal{D}=\left\{D(x)=\sig(w^Tx+b): w\in\mathbb{R}^p,b\in\mathbb{R}\right\}. \label{eq:TV-NN1}
\end{equation}
With $D(x)=\sig(w^Tx+b)$ in (\ref{eq:TV-NN1}), the procedure (\ref{eq:TV-GAN-NN1}) can be viewed as a smoothed version of TV-Learning (\ref{eq:TV-Learning}). To be specific, the sigmoid function $\sig(w^Tx+b)$ tends to an indicator function as $\|w\|\rightarrow\infty$, which leads to a procedure very similar to (\ref{eq:TV-Learning-finite-r}). In fact, the class (\ref{eq:TV-NN1}) is richer than the one used in (\ref{eq:TV-Learning-finite-r}), and thus (\ref{eq:TV-GAN-NN1}) can be understood as the minimizer of a sharper variational lower bound than that of (\ref{eq:TV-Learning-finite-r}).

\begin{thm}\label{thm:TV-NN1}
Assume $\frac{p}{n}+\epsilon^2\leq c$ for some sufficiently small constant $c>0$.
With i.i.d. observations $X_1,...,X_n\sim (1-\epsilon)N(\theta,I_p)+\epsilon Q$, the estimator $\wh{\theta}$ defined by (\ref{eq:TV-GAN-NN1}) satisfies
$$\|\wh{\theta}-\theta\|^2 \leq C\left(\frac{p}{n}\vee\epsilon^2\right),$$
with probability at least $1-e^{-C'(p+n\epsilon^2)}$ uniformly over all $\theta\in\mathbb{R}^p$ and all $Q$. The constants $C,C'>0$ are universal.
\end{thm}

Though TV-GAN can achieve the minimax rate, it may suffer from optimization difficulties especially when the distributions $Q$ and $N(\theta, I_p)$ are far away from each other. The main obstacle is, with optimization based on gradient, the discriminator may be stuck in a local maximum which will then pass wrong signals to the generator. We illustrate this point with a simple one-dimensional example in Figure \ref{fig:tv_landscape_10}, where samples are drawn from $(1-\epsilon)N(1, 1) + \epsilon N(10, 1)$ with $\epsilon=0.2$, and we optimize (\ref{eq:TV-GAN-NN1}) via alternative gradient ascent and descent shown in Algorithm \ref{alg:jsgan}. Even with a good initialization, TV-GAN in the form of (\ref{eq:TV-GAN-NN1}) will continuously increase the value of $\eta$ (from the light area to the dark area in the heatmap) if $w$ cannot achieve its global maximum, and thus fails to learn the saddle point. However, it is almost impossible for ${w}$ to correct its way from $w\to\infty$ to $w\to-\infty$ simply by the information of its local gradient. In comparison, the landscape becomes better when $Q$ and $N(\theta, I_p)$ are close, where the signal passed to the generator becomes weak before being stuck in the local maximum, as shown in Figure \ref{fig:tv_landscape_15}.
\begin{figure}
\includegraphics[width=.5\textwidth]{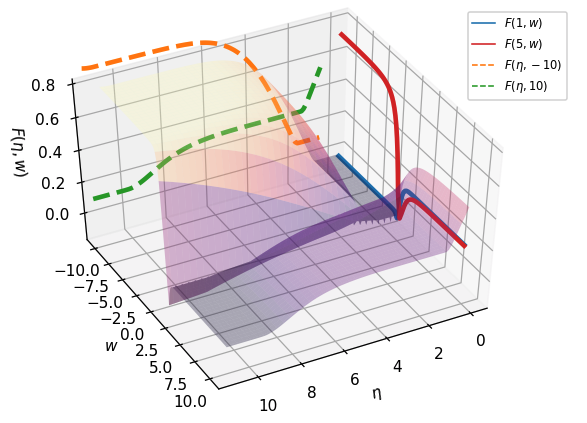}
\includegraphics[width=.5\textwidth]{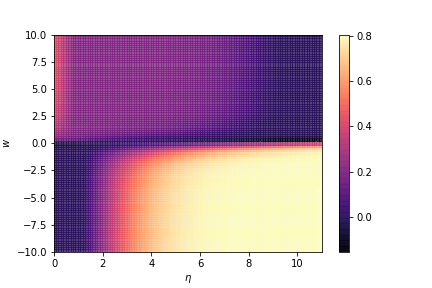}
\caption{Landscape of $F(\eta, w)= \sup_{b}[E_P\sig(wX+b)-E_{N(\eta,1)}\sig(wX+b)]$, where $b$ is maximized out for visualization. Samples are drawn from $P=(1-\epsilon)N(1, 1) + \epsilon N(10, 1)$ with $\epsilon=0.2$. Left: a surface plot of $F(\eta,w)$. The solid curves are marginal functions for fixed $\eta$'s: $F(1, w)$ (red) and $F(5, w)$ (blue), and the dash curves are marginal functions for fixed $w$'s: $F(\eta, -10)$ (orange) and $F(\eta, 10)$ (green). Right: a heatmap of $F(\eta, w)$. It is clear that $\tilde{F}(w) = F(\eta, w)$ has two local maxima for a given $\eta$, achieved at $w=+\infty$ and $w=-\infty$. In fact, the global maximum for $\tilde{F}(w)$ has a phase transition from $w=+\infty$ to $w=-\infty$ as $\eta$ grows. For example, the maximum is achieved at $w=+\infty$ when $\eta=1$ (blue solid) and is achieved at $w=-\infty$ when $\eta=5$ (red solid). Unfortunately, even if we initialize with ${\eta}_0=1$ and ${w}_0> 0$, gradient ascents on $\eta$ will only increase the value of  $\eta$ (green dash),  and thus as long as the discriminator cannot reach the global maximizer, $w$ will be stuck in the positive half space $\{w: w>0\}$ and further increase the value of $\eta$.}\label{fig:tv_landscape_10} 
\end{figure} 
\begin{figure}
\includegraphics[width=.5\textwidth]{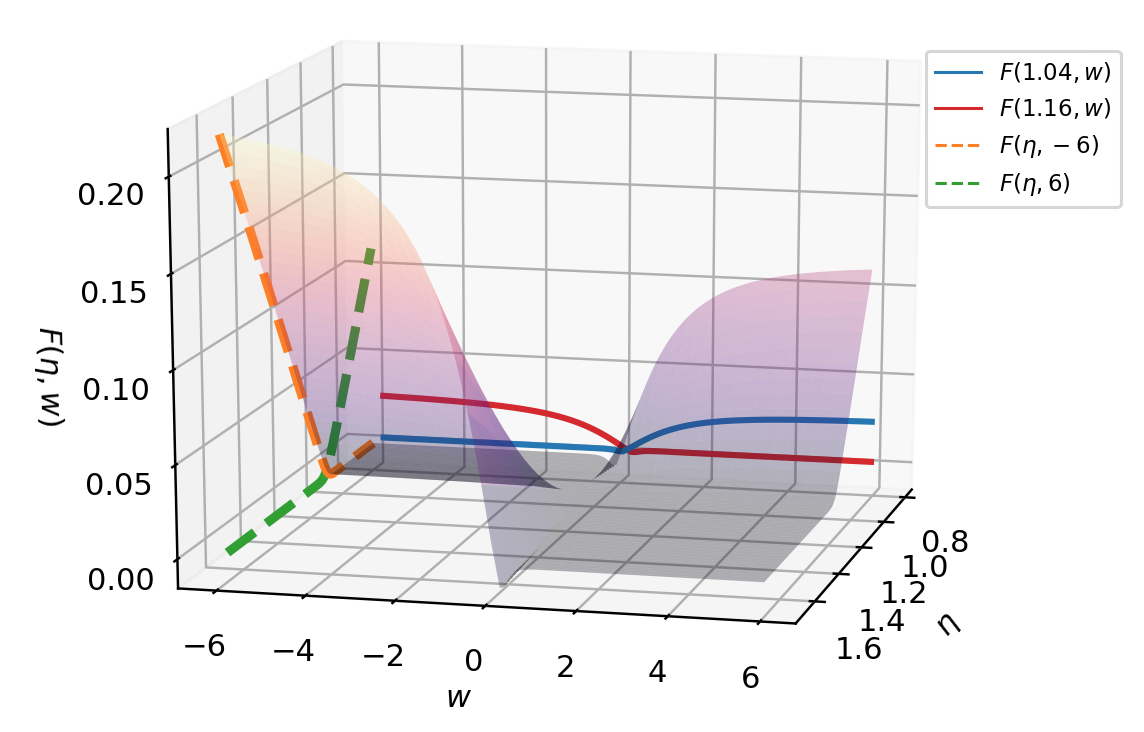}
\includegraphics[width=.5\textwidth]{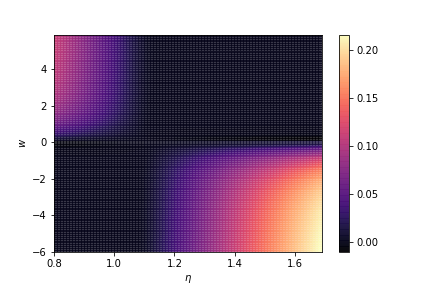}
\caption{Landscape of $F(\eta, w)= \sup_{b}[E_P\sig(wX+b)-E_{N(\eta,1)}\sig(wX+b)]$, where $b$ is maximized out for visualization. Samples are drawn from $P=(1-\epsilon)N(1, 1) + \epsilon N(1.5, 1)$ with $\epsilon=0.2$. Left: a surface plot of $F(\eta,w)$. Right: a heatmap of $F(\eta, w)$. Compared with the heatmap in Figure \ref{fig:tv_landscape_10}, the landscape becomes better in the sense that no matter whether we start from the left-top area or the right-bottom area of the heatmap, gradient ascent on $\eta$ does not consistently increase or decrease the value of $\eta$. This is because the signal becomes weak when it is close to the saddle point around $\eta=1$.}\label{fig:tv_landscape_15}
\end{figure} 


\subsection{Results for JS-GAN}

Given the intractable optimization property of TV-GAN, we next turn to Jensen-Shannon GAN (JS-GAN) with
$$f(x)=x\log x-(x+1)\log\frac{x+1}{2}.$$
The estimator is defined by
\begin{equation}
\wh{\theta}=\arginf_{\eta\in\mathbb{R}^p}\sup_{D\in\mathcal{D}}\left[\frac{1}{n}\sum_{i=1}^n\log D(X_i) + E_{N(\eta,I_p)}\log(1-D(X))\right]+\log 4, \label{eq:JS-GAN-gen}
\end{equation}
with $T(x)=\log D(x)$ in (\ref{eq:f-GAN}).
This is exactly the original GAN \citep{goodfellow2014generative} specialized to the normal mean estimation problem. The advantages of JS-GAN over other forms of GAN have been studied extensively in the literature \citep{lucic2017gans,kurach2018gan}.

Before presenting theoretical properties of (\ref{eq:JS-GAN-gen}), we first show a simple numerical result that implies important consequences on the choice of the discriminator class $\mathcal{D}$. Consider i.i.d. observations drawn from the one-dimensional contamination model $(1-\epsilon)N(\theta,1)+ \epsilon N(t,1)$ with $\theta=1$ and $\epsilon=0.2$. We consider two estimators in the form of (\ref{eq:JS-GAN-gen}) that use different discriminator classes. The first one is the same logistic regression class defined in (\ref{eq:TV-NN1}), and the second one is the class of neural networks with one hidden layer. Then, the values of the two estimators are plotted against $t$ in Figure \ref{fig:js_landscape}. It is clear that the two estimators have completely different behaviors. For the estimator trained by JS-GAN using a logistic regression discriminator class, it is always close to $0.2+0.8t$, which is the grand mean of the entire distribution $(1-\epsilon)N(\theta,1)+ \epsilon N(t,1)$. Thus, the estimator is not robust, and its deviation from $\theta$ will become arbitrarily large when the value of $t$ is increased. On the other hand, with an extra hidden layer built into the neural nets, the second estimator is always close to the mean $\theta$ that we want to learn, regardless of the value of $t$. The green curve in Figure \ref{fig:js_landscape} first increases as $t$ increases, but it eventually converges to $\theta=1$ as $t$ further increases. The hardest contamination distribution $N(t,1)$ is the one with a $t$ that is not far away from $\theta$, which is well predicted by the minimax theory of robust estimation \citep{chen2018robust}.
\begin{figure}[H]
\centering
\includegraphics[width=.50\textwidth]{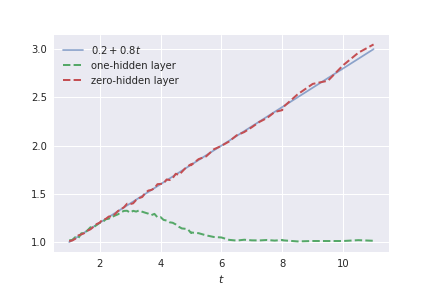}
\caption{The solid blue line is the mean of $(1-\epsilon)N(\theta,1)+ \epsilon N(t,1)$ with $\theta=1$ and $\epsilon=0.2$. At each level of $t$, we consider the estimators in the form of (\ref{eq:JS-GAN-gen}) that use different discriminator classes. The JS-GAN using discriminators without hidden layers always gives an estimator close to $0.2+0.8t$ (green dash line), while the JS-GAN using discriminators with one hidden layer leads to robust estimation (red dash line).}\label{fig:js_landscape} 
\end{figure}

To understand why and how the class of the discriminators affects the robustness property of JS-GAN, we introduce a new concept called restricted Jensen-Shannon divergence. Let $g:\mathbb{R}^p\rightarrow\mathbb{R}^d$ be a function that maps a $p$-dimensional observation to a $d$-dimensional feature space. The restricted Jensen-Shannon divergence between two probability distributions $P$ and $Q$ with respect to the feature $g$ is defined as
$$\JS_g(P,Q)=\sup_{w\in\mathcal{W}}\left[E_P\log\sig(w^Tg(X))+E_Q\log(1-\sig(w^Tg(X)))\right]+\log 4.$$
In other words, $P$ and $Q$ are distinguished by a logistic regression classifier that uses the feature $g(X)$. It is easy to see that $\JS_g(P,Q)$ is a variational lower bound of the original Jensen-Shannon divergence. The key property of $\JS_g(P,Q)$ is given by the following proposition.
\begin{proposition}\label{prop:local-JS}
Assume $\mathcal{W}$ is a convex set that contains an open neighborhood of $0$. Then, $\JS_g(P,Q)=0$ if and only if $E_Pg(X)=E_Qg(X)$.
\end{proposition}
\begin{proof}
Define $F(w)=E_P\log\sig(w^Tg(X))+E_Q\log(1-\sig(w^Tg(X)))+\log 4$, so that $\JS_g(P,Q)=\sup_{w\in\mathcal{W}}F(w)$. The gradient and Hessian of $F(w)$ are given by
\begin{eqnarray*}
\nabla F(w) &=& E_P\frac{e^{-w^Tg(X)}}{1+e^{-w^Tg(X)}}g(X) - E_Q\frac{e^{w^Tg(X)}}{1+e^{w^Tg(X)}}g(X), \\
\nabla^2 F(w) &=& -E_P \frac{e^{w^Tg(X)}}{(1+e^{w^Tg(X)})^2}g(X)g(X)^T - E_Q \frac{e^{-w^Tg(X)}}{(1+e^{-w^Tg(X)})^2}g(X)g(X)^T.
\end{eqnarray*}
Therefore, $F(w)$ is concave in $w$, and $\sup_{w\in\mathcal{W}}F(w)$ is a convex optimization with a convex $\mathcal{W}$. Suppose $\JS_g(P,Q)=0$. Then $\sup_{w\in\mathcal{W}}F(w)=0=F(0)$, which implies $\nabla F(0)=0$, and thus we have $E_Pg(X)=E_Qg(X)$. Now suppose $E_Pg(X)=E_Qg(X)$, which is equivalent to $\nabla F(0)=0$. Therefore, $w=0$ is a stationary point of a concave function, and we have $\JS_g(P,Q)=\sup_{w\in\mathcal{W}}F(w)=F(0)=0$.
\end{proof}

The proposition asserts that $\JS_g(\cdot,\cdot)$ cannot distinguish $P$ and $Q$ if the feature $g(X)$ has the same expected value under the two distributions. This \textit{generalized moment matching effect} has also been studied by \cite{liu2017approximation} for general $f$-GANs. However, the linear discriminator class considered in \cite{liu2017approximation} is parameterized in a different way compared with the discriminator class here.

When we apply Proposition \ref{prop:local-JS} to robust mean estimation, the JS-GAN is trying to match the values of $\frac{1}{n}\sum_{i=1}^ng(X_i)$ and $E_{N(\eta,I_\eta)}g(X)$ for the feature $g(X)$ used in the logistic regression classifier. This explains what we observed in our numerical experiments. A neural net without any hidden layer is equivalent to a logistic regression with a linear feature $g(X)=(X^T,1)^T\in\mathbb{R}^{p+1}$. Therefore, whenever $\eta=\frac{1}{n}\sum_{i=1}^nX_i$, we have $\JS_g\left(\frac{1}{n}\sum_{i=1}^n\delta_{X_i},N(\eta,I_p)\right)=0$, which implies that the sample mean is a global maximizer of (\ref{eq:JS-GAN-gen}). On the other hand, a neural net with at least one hidden layers involves a nonlinear feature function $g(X)$, which is the key that leads to the robustness of (\ref{eq:JS-GAN-gen}).

We will show rigorously that a neural net with one hidden layer is sufficient to make (\ref{eq:JS-GAN-gen}) robust and optimal. Consider the following class of discriminators,
\begin{equation}
\mathcal{D}=\left\{D(x)=\sig\left(\sum_{j\geq 1}w_j\sigma(u_j^Tx+b_j)\right): \sum_{j\geq 1}|w_j|\leq \kappa, u_j\in\mathbb{R}^p, b_j\in\mathbb{R}\right\}.\label{eq:JS-NN2}
\end{equation}
The class (\ref{eq:JS-NN2}) consists of two-layer neural network functions. While the dimension of the input layer is $p$, the dimension of the hidden layer can be arbitrary, as long as the weights have a bounded $\ell_1$ norm. The nonlinear activation function $\sigma(\cdot)$ is allowed to take 1) indicator: $\sigma(x)=\mathbb{I}\{x\geq 1\}$, 2) sigmoid: $\sigma(x)=\frac{1}{1+e^{-x}}$, 3) ramp: $\sigma(x)=\max(\min(x+1/2,1),0)$. Other bounded activation functions are also possible, but we do not exclusively list them. The rectified linear unit (ReLU) will be studied in Section \ref{sec:deep}.

\begin{thm}\label{thm:JS-NN2}
Consider the estimator $\wh{\theta}$ defined by (\ref{eq:JS-GAN-gen}) with $\mathcal{D}$ specified by (\ref{eq:JS-NN2}). Assume $\frac{p}{n}+\epsilon^2\leq c$ for some sufficiently small constant $c>0$, and set $\kappa=O\left(\sqrt{\frac{p}{n}}+\epsilon\right)$.
With i.i.d. observations $X_1,...,X_n\sim (1-\epsilon)N(\theta,I_p)+\epsilon Q$, we have
$$\|\wh{\theta}-\theta\|^2 \leq C\left(\frac{p}{n}\vee\epsilon^2\right),$$
with probability at least $1-e^{-C'(p+n\epsilon^2)}$ uniformly over all $\theta\in\mathbb{R}^p$ and all $Q$. The constants $C,C'>0$ are universal.
\end{thm}

Theorem \ref{thm:JS-NN2} verifies our numerical experiments, and shows that the JS-GAN using a neural net discriminator with hidden layers is not only robust, but it also achieves the minimax rate of the problem. The condition $\kappa=O\left(\sqrt{\frac{p}{n}}+\epsilon\right)$ is needed for technical reasons, and the numerical performance does not seem to be affected without it. Figure \ref{fig:rate-1} shows numerical experiments with i.i.d. observations drawn from $(1-\epsilon)N(0_p,I_p)+ \epsilon N(t*1_p,I_p)$ with $\epsilon=0.2$. The magnitude of $t$ characterizes the distance between $N(0_p,I_p)$ and the contamination distribution $N(t*1_p,I_p)$. When $t$ is very small, the contamination barely affects the overall distribution, and we expect a good performance of the estimator. On other hand, when $t$ is very large, it is easy to tell the difference between the contaminated observations and the good ones. Therefore, the hardest case is when $t$ is close to $0$, but not too close, which is verified by the left plot of Figure \ref{fig:rate-1}. The right plot of Figure \ref{fig:rate-1} demonstrates the relation between $\|w\|_1$ and the value of $t$. Note that a larger value of $\|w\|_1$ indicates that it is easier to tell the difference between the data generating process $(1-\epsilon)N(\theta,I_p)+\epsilon Q$ and the distribution we learned, which is $N(\wh{\theta},I_p)$. Therefore, we observe an increasing pattern of $\|w\|_1$ with respect to $t$ in Figure \ref{fig:rate-1}. If we imposed a constraint on $\|w\|_1$ in the optimization, the JS-GAN would have a less distinguishing ability between the data generating process and the estimated model, which would further affect the performance of the estimator when $t$ is very large (the error would not eventually decrease as in the left plot in Figure \ref{fig:rate-1}). In summary, the $\ell_1$ constraint is only needed in the proof to establish the minimax (worst-case) convergence rate, but it is not needed in practice so that the estimator can perform even better than the minimax rate when the contamination distribution is far away from $N(\theta,I_p)$.
\begin{figure}[H]
\includegraphics[width=.48\textwidth]{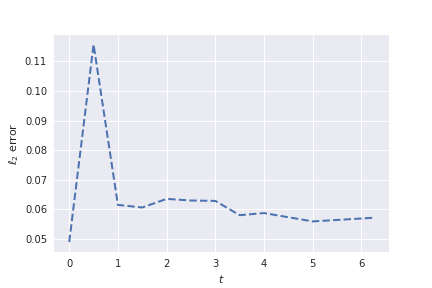}
\includegraphics[width=.48\textwidth]{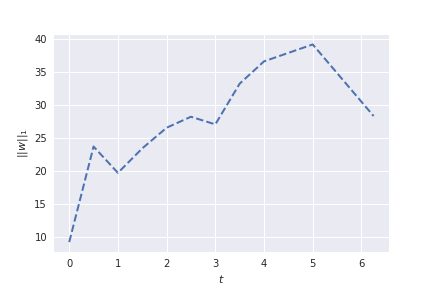}
\caption{Numerical experiments for JS-GAN with $p=100$ and $n=50,000$. Left: $\ell_2$ error with respect to $t$. Right: the $\ell_1$ norm $\|w\|_1$ of the weight matrix in the last layer with respect to $t$. Network structure: 100-20-1.}\label{fig:rate-1}
\end{figure} 

\section{Deep ReLU Networks}\label{sec:deep}

In this section, we investigate the performance of discriminator classes of deep neural nets with the ReLU activation function. 
Since our goal is to learn a $p$-dimensional mean vector, a deep neural network discriminator without any regularization will certainly lead to overfitting. Therefore, it is crucial to design a network class with some appropriate regularizations. Inspired by the work of \cite{bartlett1997valid,bartlett2002rademacher}, we consider a network class with $\ell_1$ regularizations on all layers except for the second last layer with an $\ell_2$ regularization.
With $\mathcal{G}_1^H(B)=\left\{g(x)=\relu(v^Tx): \|v\|_1\leq B\right\}$, a neural network class with $l+1$ layers is defined as
$$\mathcal{G}_{l+1}^H(B)=\left\{g(x)=\relu\left(\sum_{h=1}^Hv_hg_h(x)\right): \sum_{h=1}^H|v_h|\leq B, g_h\in\mathcal{G}_l^H(B)\right\}.$$
Combining with the last sigmoid layer, we obtain the following discriminator class,
\begin{eqnarray*}
{\mathcal{F}}_L^H(\kappa,\tau,B) &=& \Bigg\{D(x)=\sig\left(\sum_{j\geq 1}w_j\sig\left(\sum_{h=1}^{2p}u_{jh}g_{jh}(x)+b_j\right)\right): \\
&&\quad\quad \sum_{j\geq 1}|w_j|\leq\kappa, \sum_{h=1}^{2p}u_{jh}^2\leq 2,|b_j|\leq\tau, g_{jh}\in\mathcal{G}_{L-1}^H(B)\Bigg\}.
\end{eqnarray*}
Note that all the activation functions are $\relu(\cdot)$ except that we use $\sig(\cdot)$ in the last layer of feature map $g(\cdot)$.
A theoretical guarantees of the class defined above is given by the following theorem.
\begin{thm}\label{thm:JS-DNN}
Assume $\frac{p\log p}{n}\vee\epsilon^2\leq c$ for some sufficiently small constant $c>0$.
Consider i.i.d. observations $X_1,...,X_n\sim (1-\epsilon)N(\theta,I_p)+\epsilon Q$ and the estimator $\wh{\theta}$ defined by (\ref{eq:JS-GAN-gen}) with $\mathcal{D}=\mathcal{F}_L^H(\kappa,\tau,B)$ with $H\geq 2p$, $2\leq L=O(1)$, $2\leq B=O(1)$, and $\tau=\sqrt{p\log p}$.
We set $\kappa= O\left(\sqrt{\frac{p\log p}{n}}+\epsilon\right)$. Then, for the estimator $\wh{\theta}$ defined by (\ref{eq:JS-GAN-gen}) with $\mathcal{D}={\mathcal{F}}_L^H(\kappa,\tau,B)$, we have
$$
\|\wh{\theta}-\theta\|^2\leq C\left(\frac{p\log p}{n}\vee\epsilon^2\right),
$$
with probability at least $1-e^{-C'(p\log p+n\epsilon^2)}$ uniformly over all $\theta\in\mathbb{R}^p$ such that $\|\theta\|_{\infty}\leq \sqrt{\log p}$ and all $Q$.
\end{thm}

The theorem shows that JS-GAN with a deep ReLU network can achieve the error rate $\frac{p\log p}{n}\vee\epsilon^2$ with respect to the squared $\ell_2$ loss. The condition $\|\theta\|_{\infty}\leq\sqrt{\log p}$ for the ReLU network can be easily satisfied with a simple preprocessing step. We split the data into two halves, whose sizes are $\log n$ and $n-\log n$, respectively. Then, we calculate the coordinatewise median $\wt{\theta}$ using the small half. It is easy to show that $\|\wt{\theta}-\theta\|_{\infty}\leq \sqrt{\frac{\log p}{\log n}}\vee\epsilon$ with high probability. Then, for each $X_i$ from the second half, the conditional distribution of $X_i-\wt{\theta}$ given the first half is $(1-\epsilon)N(\theta-\wt{\theta},I_p)+\epsilon \wt{Q}$. Since $\sqrt{\frac{\log p}{\log n}}\vee\epsilon\leq \sqrt{\log p}$, the condition $\|\theta-\wt{\theta}\|_{\infty}\leq\sqrt{\log p}$ is satisfied, and thus we can apply the estimator (\ref{eq:JS-GAN-gen}) using the shifted data $X_i-\wt{\theta}$ from the second half. The theoretical guarantee of Theorem \ref{thm:JS-DNN} will be
$$\|\wh{\theta}-(\theta-\wt{\theta})\|^2\leq C\left(\frac{p\log p}{n}\vee\epsilon^2\right),$$
with high probability. Hence, we can use $\wh{\theta}+\wt{\theta}$ as the final estimator to achieve the same rate in Theorem \ref{thm:JS-DNN}.

On the other hand, our experiments show that this preprocessing step is not needed. We believe that the assumption $\|\theta\|_{\infty}\leq\sqrt{\log p}$ is a technical artifact in the analysis of the Rademacher complexity. It can probably be dropped by a more careful analysis.

\section{Elliptical Distributions} \label{sec:ellip}

An advantage of Tukey's median (\ref{eq:Tukey}) is that it leads to optimal robust location estimation under general elliptical distributions including Cauchy distribution whose mean does not exist. In this section, we show that JS-GAN shares the same property.
A random vector $X\in\mathbb{R}^p$ follows an elliptical distribution if it admits a representation
$$X=\theta+\xi AU,$$
where $U$ is uniformly distributed on the unit sphere $\{u\in\mathbb{R}^p:\|u\|=1\}$ and $\xi\geq 0$ is a random variable independent of $U$ that determines the shape of the elliptical distribution \citep{fang2017symmetric}. The center and the scatter matrix are $\theta$ and $\Sigma=AA^T$.

For a unit vector $v$, let the density function of $\xi v^TU$ be $h$. Note that $h$ is independent of $v$ because of the symmetry of $U$. Then, there is a one-to-one relation between the distribution of $\xi$ and $h$, and thus the triplet $(\theta,\Sigma,h)$ fully parametrizes an elliptical distribution.

Note that $h$ and $\Sigma=AA^T$ are not identifiable, because $\xi A=(c\xi)(c^{-1}A)$ for any $c>0$. Therefore, without loss of generality, we can restrict $h$ to be a member of the following class
$$\mathcal{H}=\left\{h:h(t)=h(-t),h\geq 0, \int h=1, \int \sigma(t)(1-\sigma(t))h(t)dt=1\right\}.$$
This makes the parametrization $(\theta,\Sigma,h)$ of an elliptical distribution fully identifiable, and we use $EC(\theta,\Sigma,h)$ to denote an elliptical distribution parametrized in this way.

The JS-GAN estimator is defined as
\begin{equation}
(\wh{\theta},\wh{\Sigma},\wh{h})=\arginf_{\eta\in\mathbb{R}^p,\Gamma\in\mathcal{E}_p(M),g\in\mathcal{H}}\sup_{D\in\mathcal{D}}\left[\frac{1}{n}\sum_{i=1}^n\log D(X_i)+E_{EC(\eta,\Gamma,g)}\log(1-D(X))\right]+\log 4,\label{eq:JS-GAN-EC}
\end{equation}
where $\mathcal{E}_p(M)$ is the set of all positive semi-definite matrix with spectral norm bounded by $M$.

\begin{thm}\label{thm:elliptical}
Consider the estimator $\wh{\theta}$ defined above with $\mathcal{D}$ specified by (\ref{eq:JS-NN2}). Assume $M=O(1)$, $\frac{p}{n}+\epsilon^2\leq c$ for some sufficiently small constant $c>0$, and set $\kappa=O\left(\sqrt{\frac{p}{n}}+\epsilon\right)$. With i.i.d. observations $X_1,...,X_n\sim (1-\epsilon)EC(\theta,\Sigma,h)+\epsilon Q$, we have
$$\|\wh{\theta}-\theta\|^2\leq C\left(\frac{p}{n}\vee\epsilon^2\right),$$
with probability at least $1-e^{-C'(p+n\epsilon^2)}$ uniformly over all $\theta\in\mathbb{R}^p$, $\Sigma\in\mathcal{E}_p(M)$ and all $Q$. The constants $C,C'>0$ are universal.
\end{thm}

Note that Theorem \ref{thm:elliptical} guarantees the same convergence rate as in the Gaussian case for all elliptical distributions. This even includes multivariate Cauchy where mean does not exist. Therefore, the location estimator (\ref{eq:JS-GAN-EC}) is fundamentally different from \cite{diakonikolas2016robust,lai2016agnostic}, which is only designed for robust mean estimation.

To achieve rate-optimality for robust location estimation under general elliptical distributions, the estimator (\ref{eq:JS-GAN-EC}) is different from (\ref{eq:JS-GAN-gen}) only in the generator class. They share the same discriminator class (\ref{eq:JS-NN2}). This underlines an important principle for designing GAN estimators: the overall statistical complexity of the estimator is only determined by the discriminator class.

The estimator (\ref{eq:JS-GAN-EC}) also outputs $(\wh{\Sigma},\wh{h})$, but we do not claim any theoretical property for $(\wh{\Sigma},\wh{h})$ in this paper.

\section{Numerical Experiments} \label{sec:num}

In this section, we give extensive numerical studies of robust mean estimation via GAN. After introducing the implementation details in Section \ref{sec:imp}, we verify our theoretical results on minimax estimation with both TV-GAN and JS-GAN in Section \ref{sec:rate}. Comparison with other methods on robust mean estimation in the literature is given in Section \ref{sect:comp}. The effects of various network structures are studied in Section \ref{sec:n-structure}. Finally, adaptation to unknown covariance structure and elliptical distributions are investigated in Section \ref{sec:cov} and Section \ref{sec:last}.

\subsection{Implementations}\label{sec:imp}
The implementation for JS-GAN is given in Algorithm \ref{alg:jsgan}, and a simple modification of the objective function leads to that of TV-GAN. A PyTorch implementation is available at \url{https://github.com/zhuwzh/Robust-GAN-Center} or \url{https://github.com/yao-lab/Robust-GAN-Center}.
Several important implementation details are listed below.
\begin{algorithm}
\caption{JS-GAN: $\arginf_{\eta}\sup_{w}[\frac{1}{n}\sum_{i=1}^n \log D_{w}(X_i) + \mathbb{E} \log(1-D_{w}(G_{\eta}(Z)))]$}\label{alg:jsgan}
\textbf{Input}: Observation set $S=\{X_{1},\ldots,X_{n}\}\in\mathbb{R}^{p}$, discriminator network $D_{w}(x)$, generator network $G_{\eta}(z)=z+\eta$, learning rates $\gamma_{d}$ and $\gamma_{g}$ for the discriminator and the generator, batch size $m$, discriminator steps in each iteration $K$, total epochs $T$, average epochs $T_0$.\\
\textbf{Initialization}: Initialize $\eta$ with coordinatewise median of $S$. Initialize $w$ with $N(0, .05)$ independently on each element or Xavier \citep{glorot2010understanding}.
\begin{algorithmic}[1]
	\For{\texttt{$t=1,\ldots,T$}}
		\For{\texttt{$k=1,\ldots,K$}}
			\State Sample mini-batch $\{X_{1},\ldots,X_{m}\}$ from $S$. Sample $\{Z_{1},\ldots,Z_{m}\}$ from $N(0,I_{p})$
			\State $g_{w}\gets\nabla_{w}[\frac{1}{m}\Sigma_{i=1}^{m} \log D_{w}(X_{i}) + \frac{1}{m}\Sigma_{i=1}^{m} \log (1 - D_{w}(G_{\eta}(Z_{i})))]$ 
			\State $w \gets w + \gamma_{d}g_{w}$
		\EndFor
		\State Sample $\{Z_{1},\ldots,Z_{m}\}$ from $N(0,I_{p})$
		\State $g_{\eta}\gets\nabla_{\eta}[\frac{1}{m}\Sigma_{i=1}^{m} \log(1-D_{w}(G_{\eta}(Z_{i})))]$
		\State $\eta \gets \eta - \gamma_{g}g_{\eta}$
	\EndFor
\end{algorithmic}
\textbf{Return}: The average estimate $\eta$ over the last $T_0$ epochs.
\end{algorithm}

\begin{itemize}
\item \textit{How to tune parameters?} The choice of learning rates is crucial to the convergence rate, but the minimax game is hard to evaluate. We propose a simple strategy to tune hyper-parameters including the learning rates. Suppose we have estimators $\wh{\theta}_{1},\ldots, \wh{\theta}_{M}$ with corresponding discriminator networks $D_{\wh{w}_{1}}$,\ldots, $D_{\wh{w}_{M}}$. Fixing $\eta=\wh{\theta}$, we further apply gradient descent to $D_{w}$ with a few more epochs (but not many in order to prevent overfitting, for example 10 epochs) and select the $\wh{\theta}$ with the smallest value of the objective function (\ref{eq:JS-GAN-gen}) (JS-GAN) or (\ref{eq:TV-GAN-NN1}) (TV-GAN). We note that training discriminator and generator alternatively usually will not suffer from overfitting since the objective function for either the discriminator or the generator is always changing. However, we must be careful about the overfitting issue when training the discriminator alone with a fixed $\eta$, and that is why we apply an early stopping strategy here. Fortunately, the experiments show that if the structures of networks are same (then of course, the dimensions of the inputs are same), the choices of hyper-parameters are robust to different models.

\item \textit{When to stop training?} Judging convergence is a difficult task in GAN trainings, since sometimes oscillation may occur. In computer vision, people often use a task related measure and stop training once the requirement based on the measure is achieved. In our experiments below, we simply use a sufficiently large $T$ (see below), which works well in practice. It is interesting to explore an efficient early stopping rule in the future work.
\item \textit{How to design the network structure?} Although Theorem \ref{thm:TV-NN1} and Theorem \ref{thm:JS-NN2} guarantee the minimax rates of TV-GAN without hidden layer and JS-GAN with one hidden layer, one may wonder whether deeper network structures will perform better. From our experiments, TV-GAN with one hidden layer is better than TV-GAN without any hidden layer. Moreover, JS-GAN with deep network structures can significantly improve over shallow networks especially when the dimension is large (e.g. $p\geq 200$). For a network with one hidden layer, the choice of width may depend on the sample size. If we only have 5,000 samples of 100 dimensions, two hidden units performs better than five hidden units, which performs better than twenty hidden units. If we have 50,000 samples, networks with twenty hidden units perform the best.
\item \textit{How to stabilize and accelerate TV-GAN?} As we have discussed in Section \ref{sec:TV-GAN}, TV-GAN has a bad landscape when $N(\theta,I_p)$ and the contamination distribution $Q$ are linearly separable (see Figure \ref{fig:tv_landscape_10}). An outlier removal step before training TV-GAN may be helpful. Besides, spectral normalization \citep{miyato2018spectral} is also worth trying since it can prevent the weight from going to infinity and thus can increase the chance to escape from bad saddle points. To accelerate the optimization of TV-GAN, in all the numerical experiments below, we adopt a regularized version of TV-GAN inspired by Proposition \ref{prop:local-JS}. Since a good feature extractor should match nonlinear moments of $P=(1-\epsilon)N(\theta,I_p) + \epsilon Q$ and $N(\eta, I_{p})$, we use an additional regularization term that can accelerate training and sometimes even leads to better performances. Specifically, let $D(x)=\sig(w^T\Phi(x))$ be the discriminator network with $w$ being the weights of the output layer and $\Phi_D(x)$ be the corresponding network after removing the output layer from $D(x)$. The quantity $\Phi_D(x)$ is usually viewed as a feature extractor, which naturally leads to the following regularization term \citep{salimans2016improved,mroueh2017mcgan}, defined as
\begin{equation}\label{eq:alg-reg}
r(D,\eta) = \left\|T(\Phi_D, \mathbb{P}_n) - T(\Phi_D, N(\eta, I_{p}))\right\|^{2},
\end{equation}
where $\mathbb{P}_n=(1/n)\sum_{i=1}^n\delta_{X_i}$ is the empirical distribution, and $T(\Phi, P)$ can be either moment matching $T(\Phi, P)=\mathbb{E}_{P}\Phi(X)$, or median matching $T(\Phi, P)=\textnormal{Median}_{X\sim P}\Phi_D(X)$.
\end{itemize}

\subsection{Numerical Supports for the Minimax Rates}\label{sec:rate}

In this section,
we verify the minimax rates achieved by TV-GAN (Theorem \ref{thm:TV-NN1}) and JS-GAN (Theorem \ref{thm:JS-NN2}) via numerical experiments. The TV-GAN has no hidden layer, while the JS-GAN has one hidden layer with five hidden units in our experiments. All activation functions are sigmoid. Two main scenarios we consider here are $\sqrt{p/n} < \epsilon$ and $\sqrt{p/n}>\epsilon$, where in both cases, various types of contamination distributions $Q$, are considered. 

We introduce the contamination distributions $Q$ used in the experiments. We first consider $Q=N(\mu,I_p)$ with $\mu$ ranges in $\{0.2,0.5,1,5\}$. Note that the total variation distance between $N(0_p,I_p)$ and $N(\mu,I_p)$ is of order $\|0_p-\mu\|=\|\mu\|$. We hope to use different levels of $\|\mu\|$ to test the algorithm and verify the error rate in the worst case. Second, we consider $Q=N(1.5*1_p, \Sigma)$ to be a Gaussian distribution with a non-trivial covariance matrix $\Sigma$. The covariance matrix is generated according to the following steps. First generate a sparse precision matrix $\Gamma=(\gamma_{ij})$ with each entry $\gamma_{ij} = z_{ij}*\tau_{ij}, i\leq j$, where $z_{ij}$ and $\tau_{ij}$ are independently generated from Uniform$(0.4, 0.8)$ and Bernoulli$(0.1)$. We then define $\gamma_{ij}=\gamma_{ji}$ for all $i>j$ and $\bar{\Gamma}=\Gamma + (|\min \textnormal{eig}(\Gamma)| + 0.05)I_p$ to make the precision matrix symmetric and positive definite, where $\min\textnormal{eig}(\Gamma)$ is the smallest eigenvalue of $\Gamma$. The covariance matrix is $\Sigma=\bar{\Gamma}^{-1}$. Finally, we consider $Q$ to be a Cauchy distribution with independent component, and the $j$th component takes a standard Cauchy distribution with location parameter $\tau_j=0.5$.

Tables \ref{tab:rate-eps}-\ref{tab:rate-n} show experiment results with i.i.d. samples drawn from $(1-\epsilon)N(0_p,I_p)+\epsilon Q$.
The first scenario we consider is when $\epsilon$ dominates $\sqrt{p/n}$, and we expect the worse-case $\ell_2$ loss $\|\wh{\theta}-\theta\|$ is approximately linear with respect to $\epsilon$. Table \ref{tab:rate-eps} shows the performance of both JS-GAN and TV-GAN. To visualize the verification of the minimax rate, we take the maximum error among all choices of $Q$ in Table \ref{tab:rate-eps}, and plot the worst-case errors in Figure \ref{fig:rate}. Similar experiments are conducted for the second scenario when $\sqrt{p/n}$ dominates $\epsilon$. Table \ref{tab:rate-p} and Table \ref{tab:rate-n} show experiment results with a fixed $n$ and a fixed $p$, respectively. Again, the worst-case errors among all $Q$'s considered are plotted in Figure \ref{fig:rate}.
\begin{table}[H]
\centering
   \resizebox{\linewidth}{!}{%
   \begin{tabular}{|c|c|c|c|c|c|}
   	\hline
	$Q$ & Net & $\epsilon=0.05$ & $\epsilon=0.10$ & $\epsilon=0.15$ & $\epsilon=0.20$\\
   	\hline \hline
	\multirow{2}{*}{$N(0.2*1_p,I_p)$} & JS & 0.1025 (0.0080) & 0.1813 (0.0122) & \textbf{0.2632 (0.0080)} & \textbf{0.3280 (0.0069)} \\
		& TV & 0.1110 (0.0204) & 0.2047 (0.0112) & 0.2769 (0.0315) & 0.3283 (0.0745) \\
	\hline
	\multirow{2}{*}{$N(0.5*1_p, I_p)$} & JS & 0.1407 (0.0061) & 0.1895 (0.0070) & 0.1714 (0.0502) & 0.1227 (0.0249)\\
		& TV & 0.2003 (0.0480) & 0.2065 (0.1495) & 0.2088 (0.0100) & 0.3985 (0.0112)\\
	\hline
	\multirow{2}{*}{$N(1_p, I_p)$} & JS & 0.0855 (0.0054) & 0.1055 (0.0322) & 0.0602 (0.0133) & 0.0577 (0.0029) \\
	& TV  & 0.1084 (0.0063) & 0.0842 (0.0036) & 0.3228 (0.0123) & 0.1329 (0.0125)\\
	\hline
	\multirow{2}{*}{$N(5*1_p, I_p)$} & JS & 0.0587 (0.0033) & 0.0636 (0.0025) & 0.0625 (0.0045) & 0.0591 (0.0040) \\
	& TV & \textsl{1.2886 (0.5292)} & \textsl{4.4511 (0.8754)} & \textsl{7.3868 (0.8081)} & \textsl{10.5724 (1.2605)}\\
	\hline
	\multirow{2}{*}{Cauchy$(0.5*1_p)$} & JS & 0.0625 (0.0045) & 0.0652 (0.0044) & 0.0648 (0.0035) & 0.0687 (0.0042)\\
	& TV & 0.2280 (0.0067) & 0.3842 (0.0083) & 0.5740 (0.0071) & \textbf{0.7768 (0.0074)}\\
	\hline
	\multirow{2}{*}{$N(0.5*1_p,\Sigma)$} & JS & \textbf{0.1490 (0.0061)} & \textbf{0.1958 (0.0074)} & 0.2379 (0.0076) & 0.1973 (0.0679) \\
	& TV & \textbf{0.2597 (0.0090)} & \textbf{0.4621 (0.0649)} & \textbf{0.6344 (0.0905)} & 0.7444 (0.3115) \\
	\hline
   \end{tabular}}
\caption{Scenario I: $\sqrt{p/n}<\epsilon$. Setting: $p=100, n=50,000$, and $\epsilon$ from $0.05$ to $0.20.$ Network structure of JS-GAN: one hidden layer with 5 hidden units. Network structure of TV-GAN: zero-hidden layer. The number in each cell is the average of $\ell_2$ loss $\|\wh{\theta}-\theta\|$ with standard deviation in parenthesis from 10 repeated experiments. The bold character marks the worst case among our choices of $Q$ at each $\epsilon$ level. The results of TV-GAN for $Q=N(5*1_p,I_{p})$ are highlighted in slanted font. The failure of training in this case is due to the bad landscape when $N(0_p,I_p)$ and $Q$ are linearly separable, as discussed in Section \ref{sec:TV-GAN} (see Figure \ref{fig:tv_landscape_10}).}\label{tab:rate-eps}
\end{table}

\begin{table}[H]
\centering
  	 \resizebox{\linewidth}{!}{%
	 \begin{tabular}{|c|c|c|c|c|c|c|}
	\hline
   	$Q$ & Net & $p=10$ & $p=25$ & $p=50$ & $p=75$ & $p=100$\\
   	\hline\hline
	\multirow{2}{*}{$N(0.2*1_p,I_p)$} & JS & 0.1078 (0.0338) & 0.1819 (0.0215) & 0.3355 (0.0470) & 0.4806 (0.0497) & 0.5310 (0.0414)\\
		&TV & 0.2828 (0.0580) & 0.4740 (0.1181) & 0.5627 (0.0894) & 0.8217 (0.0382) & 0.8090 (0.0457) \\
	\hline
	\multirow{2}{*}{$N(0.5*1_p, I_p)$} & JS & 0.1587 (0.0438) & 0.2684 (0.0386) & \textbf{0.4213 (0.0356)} & \textbf{0.5355 (0.0634)} & 0.6825 (0.0981)\\
	& TV & 0.2864 (0.0521) & 0.5024 (0.1038) & 0.6878 (0.1146) & 0.9204 (0.0589) & 0.9418 (0.0551) \\
	\hline
	\multirow{2}{*}{$N(1_p, I_p)$} & JS & 0.1644 (0.0255) & 0.2177 (0.0480) & 0.3505 (0.0552) & 0.4740 (0.0742) & 0.6662 (0.0611) \\
	& TV & \textbf{0.3733 (0.0878)} & \textbf{0.5407 (0.0634)} & \textbf{0.9061 (0.1029)} & \textbf{1.0672 (0.0629)} & \textbf{1.1150 (0.0942)} \\
	\hline
	\multirow{2}{*}{$N(5*1_p, I_p)$} & JS & 0.0938 (0.0195) & 0.2058 (0.0218) & 0.3316 (0.0462) & 0.4054 (0.0690) & 0.5553 (0.0518) \\
	& TV & 0.3707 (0.2102) & 0.7434 (0.3313) & 1.1532 (0.3488) & 1.1850 (0.3739) & 1.3257 (0.1721) \\
	\hline
	\multirow{2}{*}{Cauchy$(0.5*1_p)$} & JS & 0.1188 (0.0263) & 0.1855 (0.0282) & 0.2967 (0.0284) & 0.4094 (0.0385) & 0.4826 (0.0479) \\
	& TV & 0.3198 (0.1543) & 0.5205 (0.1049) & 0.6240 (0.0652) & 0.7536 (0.0673) & 0.7612 (0.0613) \\
	\hline
	\multirow{2}{*}{$N(0.5*1_p,\Sigma)$} & JS & \textbf{0.1805 (0.0220)} & \textbf{0.2692 (0.0318)} & 0.3885 (0.0339) & 0.5144 (0.0547) & \textbf{0.6833 (0.1094)} \\
	& TV & 0.3036 (0.0736) & 0.5152 (0.0707) & 0.7305 (0.0966) & 0.9460 (0.0900) & 1.0888 (0.0863) \\
   	\hline
   \end{tabular}}
\caption{Scenario II-a: $\sqrt{p/n}>\epsilon$. Setting: $n=1,000, \epsilon=0.1$, and $p$ from $10$ to $100$.  Other details are the same as above. The bold character marks the worst case among our choices of $Q$ at each  level of $p$.}\label{tab:rate-p}
\end{table}

\begin{table}[H]
\centering
   \resizebox{\linewidth}{!}{%
   \begin{tabular}{|c|c|c|c|c|c|c|}
   \hline
   	$Q$ & Net & $n=50$ & $n=100$ & $n=200$ & $n=500$ & $n=1000$\\
   	\hline \hline
	\multirow{2}{*}{$N(0.2*1_p,I_p)$} & JS & 1.3934 (0.5692) & 1.0055 (0.1040) & 0.8373 (0.1335) & 0.4781 (0.0677) & 0.3213 (0.0401)  \\
		&TV & 1.9714 (0.1552) & 1.2629 (0.0882) & 0.7579 (0.0486) & 0.6640 (0.0689) & 0.6348 (0.0547)\\
	\hline
	\multirow{2}{*}{$N(0.5*1_p, I_p)$} & JS & 1.6422 (0.6822) & 1.2101 (0.2826) & 0.8374 (0.1021) & \textbf{0.5832 (0.0595)} & 0.3930 (0.0485)\\
	& TV & 1.9780 (0.2157) & 1.2485 (0.0668) & 0.8198 (0.0778) & 0.7597 (0.0456) & 0.7346 (0.0750) \\
	\hline
	\multirow{2}{*}{$N(1_p, I_p)$} & JS & 1.8427 (0.9633) & 1.2179 (0.2782) & \textbf{1.0147 (0.2170)} & 0.5586 (0.1013) & 0.3639 (0.0464) \\
	& TV & 1.9907 (0.1498) & \textbf{1.4575 (0.1270)} & \textbf{0.9724 (0.0802)} & \textbf{0.9050 (0.1479)} & \textbf{0.8747 (0.0757)} \\
	\hline
	\multirow{2}{*}{$N(5*1_p, I_p)$} & JS & \textbf{2.6392 (1.3877)} & \textbf{1.3966 (0.5370)} & 0.9633 (0.1383) & 0.5360 (0.0808) & 0.3265 (0.0336)  \\
	& TV & 2.1050 (0.3763) & 1.5205 (0.2221) & 1.1909 (0.2273) & 1.0957 (0.1390) & 1.0695 (0.2639) \\
	\hline
	\multirow{2}{*}{Cauchy$(0.5*1_p)$} & JS & 1.6563 (0.5246) & 1.0857 (0.3613) & 0.8944 (0.1759) & 0.5363 (0.0593) & 0.3832 (0.0408)  \\
	& TV & \textbf{2.1031 (0.2300)} & 1.1712 (0.1493) & 0.6904 (0.0763) & 0.6300 (0.0642) & 0.5085 (0.0662) \\
	\hline
	\multirow{2}{*}{$N(0.5*1_p,\Sigma)$} & JS &  1.2296 (0.3157) & 0.7696 (0.0786)  & 0.5892 (0.0931)  & 0.5015 (0.0831)  & \textbf{0.4085 (0.0209)}   \\
	& TV & 1.9243 (0.2079) & 1.2217 (0.0681) & 0.7939 (0.0688) & 0.7033 (0.0414) & 0.7125 (0.0490) \\
   	\hline
   \end{tabular}}
\caption{Scenario II-b: $\sqrt{p/n}>\epsilon$. Setting: $p=50, \epsilon=0.1$, and $n$ from $50$ to $1,000$. Other details are the same as above. The bold character marks the worst case among our choices of $Q$ at each  level of $n$.}\label{tab:rate-n}
\end{table}

\begin{figure}[H]
\includegraphics[width=.32\textwidth]{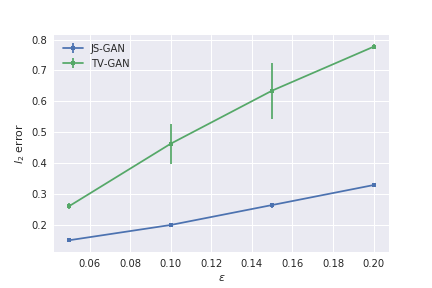}
\includegraphics[width=.32\textwidth]{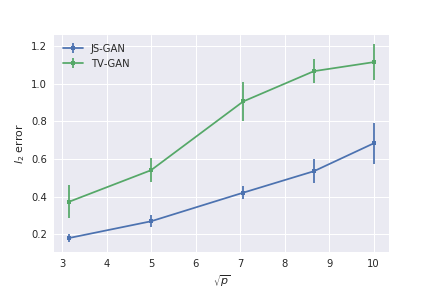}
\includegraphics[width=.32\textwidth]{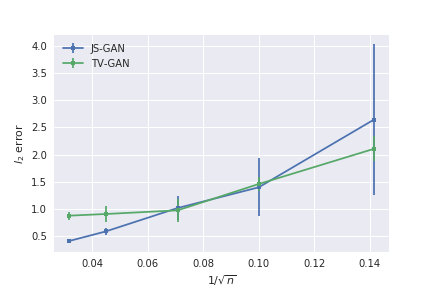}
\caption{$\ell_2$ error $\|\wh{\theta} - \theta\|$ against $\epsilon$ (left), $\sqrt{p}$ (middle) and $1/\sqrt{n}$ (right), respectively. The vertical bars indicate $\pm$ standard deviations. In all cases, the errors are approximately linear with respect to the corresponding numbers, which empirically verifies the conclusions of Theorem \ref{thm:TV-NN1} and Theorem \ref{thm:JS-NN2}.}\label{fig:rate}
\end{figure} 

\subsection{Comparisons with Other Methods}\label{sect:comp}

We perform additional experiments to compare with other methods including \textit{dimension halving} \citep{lai2016agnostic} and \textit{iterative filtering} \citep{diakonikolas2017being} under various settings. 
\begin{itemize}
\item \textit{Dimension Halving.} Experiments conducted are based on the code from \url{https://github.com/kal2000/AgnosticMeanAndCovarianceCode}. The only hyper-parameter is the threshold in the outlier removal step, and we take $C=2$ as suggested in the file \textsf{outRemSperical.m}.
\item \textit{Iterative Filtering.} Experiments conducted are based on the code from \url{https://github.com/hoonose/robust-filter}. We assume $\epsilon$ is known and take other hyper-parameters as  suggested in the file \textsf{filterGaussianMean.m}. 
\end{itemize}
We emphasize that our method does not require any prior knowledge the nuisance parameters such as the contamination proportion $\epsilon$. Tuning GAN is only a matter of optimization and one can tune parameters based on the objective function only. 

Table \ref{tab:comparison} shows the performances of JS-GAN, TV-GAN, dimension halving, and iterative filtering with i.i.d. observations sampled from $(1-\epsilon)N(0_p,I_p)+\epsilon Q$. The network structure, for both JS-GAN and TV-GAN, has one hidden layer with 20 hidden units when the sample size is 50,000 and 2 hidden units when sample size is 5,000. With fixed network structure, the hyper parameters are robust to various sampling distributions. For the network with 20 hidden units, the critical parameters to reproduce the results in the table are $\gamma_g=0.02$, $\gamma_d=0.2$, $K=5$, $T=150$ ($p=100$), $T=250$ ($p=200$), $T_0 = 25$ for JS-GAN and $\gamma_{g}=0.0001$, $\gamma_{d}=0.3$, $K=2$, $T=150$ ($p=100$), $T=250$ ($p=200$), $T_0=1$, $\lambda=0.1$ for TV-GAN, where $\lambda$ is the penalty factor of the additional regularization term (\ref{eq:alg-reg}). For the network with 2 hidden units, the critical parameters to reproduce the results below are $\gamma_g=0.01$, $\gamma_d=0.2$, $K=5$, $T=150$ ($p=100$), $T_0 = 25$ for JS-GAN and $\gamma_{g}=0.01$, $\gamma_{d}=0.1$, $K=5$, $T=150$ ($p=100$), $T_0=1$ for TV-GAN. We use Xavier initialization \citep{glorot2010understanding} for both JS-GAN and TV-GAN trainings.

To summarize, our method outperforms other algorithms in most cases. TV-GAN is good at cases when $Q$ and $N(0_p, I_p)$ are non-separable but fails when $Q$ is far away from $N(0_p, I_p)$ due to optimization issues discussed in Section \ref{sec:TV-GAN} (Figure \ref{fig:tv_landscape_10}). On the other hand, JS-GAN stably achieves the lowest error in separable cases and also shows competitive performances for non-separable ones.

\begin{table}[H]
\begin{center}
   \resizebox{\linewidth}{!}{%
   \begin{tabular}{|c|c|c|c|c|c|c|c|}
   	\hline
   	$Q$ & $n$ & $p$ & $\epsilon$ & TV-GAN & JS-GAN & Dimension Halving & Iterative Filtering \\
	\hline\hline
	$N(0.5*1_p, I_p)$  & 50,000 & 100 & .2 & \textbf{0.0953 (0.0064)} & 0.1144 (0.0154) & 0.3247 (0.0058) & 0.1472 (0.0071) \\ 
	\hline
	$N(0.5*1_p, I_p)$ & 5,000 & 100 & .2 &  \textbf{0.1941 (0.0173)} & 0.2182 (0.0527) & 0.3568 (0.0197) & 0.2285 (0.0103) \\ 
	\hline
	$N(0.5*1_p, I_p)$ & 50,000 & 200 & .2 & \textbf{0.1108 (0.0093)} & 0.1573 (0.0815) & 0.3251 (0.0078) & 0.1525 (0.0045) \\
	\hline
	$N(0.5*1_p ,I_p)$ & 50,000 & 100 &.05 & 0.0913 (0.0527) & 0.1390 (0.0050) & 0.0814 (0.0056) &\textbf{0.0530 (0.0052)}\\ 
	\hline
	$N(5*1_p, I_p)$ & 50,000 & 100 & .2 & 2.7721 (0.1285) & \textbf{0.0534 (0.0041)} & 0.3229 (0.0087) & 0.1471 (0.0059)\\ 
	\hline
	$N(0.5*1_p, \Sigma)$ & 50,000 & 100 & .2 & 0.1189 (0.0195) & \textbf{0.1148 (0.0234)} & 0.3241 (0.0088) & 0.1426 (0.0113)\\
	\hline
	Cauchy$(0.5*1_p)$ & 50,000 & 100 & .2& 0.0738 (0.0053) & \textbf{0.0525 (0.0029)} & 0.1045 (0.0071) &  0.0633 (0.0042)\\ 
	\hline
   \end{tabular}}
\caption{Comparison of various robust mean estimation methods. The smallest error of each case is highlighted in bold.}\label{tab:comparison}
\end{center}
\end{table}

\subsection{Network Structures}\label{sec:n-structure}

In this section, we study the performances of TV-GAN and JS-GAN with various structures of neural networks. The experiments are conducted with i.i.d. observations drawn from $(1-\epsilon)N(0_p, I_p)+\epsilon N(0.5*1_p, I_p)$ with $\epsilon=0.2$. Table \ref{tab:shallownet} summarizes results for $p=100$, $n\in\{5000,50000\}$ and various network structures. We observe that TV-GAN that uses neural nets with one hidden layer improves over the performance of that without any hidden layer. This indicates that the landscape of TV-GAN is improved by a more complicated network structure. However, adding one more layer does not improve the results. For JS-GAN, we omit the results without hidden layer because of its lack of robustness (Proposition \ref{prop:local-JS}). Deeper networks sometimes improve over shallow networks, but this is not always true. Table \ref{tab:deepnet} illustrates the improvements of network with more than one hidden layers over that with only one hidden layer for JS-GAN when $p\in\{200,400\}$.
We also observe that the optimal choice of the width of the hidden layer depends on the sample size.
\begin{table}[H]
   \centering
   \begin{tabular}{|c|c|c|c|c|}
   \hline
   	 Structure & $n$ & JS-GAN & TV-GAN \\
	 \hline\hline
	 100-1 & 50,000 & - & 0.1173 (0.0056) \\
	 \hline
	 100-20-1 & 50,000 & 0.0953 (0.0064) & 0.1144 (0.0154)  \\ 
	 \hline
	 100-50-1 & 50,000 & 0.2409 (0.0500) & 0.1597 (0.0219) \\
	 \hline
	 100-20-20-1 & 50,000 & 0.1131 (0.0855) & 0.1724 (0.0295)\\
	 \hline
	 100-1 & 5,000 & - & 0.9818 (0.0417)\\
	 \hline
	 100-2-1 & 5,000 & 0.1941 (0.0173) & 0.1941 (0.0173) \\
	 \hline
	 100-5-1 & 5,000 & 0.2148 (0.0241) & 0.2244 (0.0238) \\
	 \hline
	 100-20-1 & 5,000 & 0.3379 (0.0273) & 0.3336 (0.0186) \\
	\hline
   \end{tabular}
\caption{Experiment results for JS-GAN and TV-GAN with various network structures.}\label{tab:shallownet}
\end{table}

\begin{table}[H]
   \centering
   \begin{tabular}{|c|c|c|c|c|}
   \hline
   	 $p$ & 200-100-20-1 & 200-20-10-1 & 200-100-1 & 200-20-1 \\
	 \hline 
	 200 & \textbf{0.0910 (0.0056)} & 0.2251 (0.1311) & 0.3064 (0.0077) &  0.1573 (0.0815) \\
	 \hline \hline
	 $p$ & 400-200-100-50-20-1 & 400-200-100-20-1 & 400-200-20-1 & 400-200-1 \\
	 \hline
	 400 & 0.1477 (0.0053) & 0.1732 (0.0397) & \textbf{0.1393 (0.0090)} & 0.3604 (0.0990) \\
	 \hline 
   \end{tabular}
\caption{Experiment results for JS-GAN using networks with different structures.
The samples are drawn independently from $(1-\epsilon)N(0_p, I_p)+\epsilon N(0.5*1_p, I_p)$ with $\epsilon=0.2$, $p\in\{200,400\}$ and $n=50,000$.}\label{tab:deepnet}
\end{table}

\subsection{Adaptation to Unknown Covariance} \label{sec:cov}

The robust mean estimator constructed through JS-GAN can be easily made adaptive to unknown covariance structure, which is a special case of (\ref{eq:JS-GAN-EC}). We define
$$
(\wh{\theta},\wh{\Sigma}) = \arginf_{\eta\in\mathbb{R}^p,\Gamma\in\mathcal{E}_p}\sup_{D\in\mathcal{D}}\left[\frac{1}{n}\sum_{i=1}^n\log D(X_i) + E_{N(\eta,\Gamma)}\log(1-D(X_i))\right]+\log 4, $$
The estimator $\wh{\theta}$, as a result, is rate-optimal even when the true covariance matrix is not necessarily identity and is unknown (see Theorem \ref{thm:elliptical}).
	Below, we demonstrate some numerical evidence of the optimality of $\wh{\theta}$ as well as the error of $\wh{\Sigma}$ in Table \ref{tab:unknowcov}.

\begin{table}[ht]
   \centering
   \resizebox{0.9\linewidth}{!}{%
   \begin{tabular}{|c||c|c|c|c|}
   \hline
   	Data generating process & Network structure & $\|\wh{\theta}-0_p\|$ & $\opnorm{\wh{\Sigma}-\Sigma_1}$ \\
	\hline\hline
	$0.8N(0_p, \Sigma_1)+0.2N(0.5*1_p, \Sigma_2)$ & 100-20-1 & 0.1680 (0.1540) & 1.9716 (0.7405) \\ 
	\hline
	$0.8N(0_p, \Sigma_1)+0.2N(0.5*1_p, \Sigma_2)$ & 100-20-20-1 & 0.1824 (0.3034) & 1.4495 (0.6028) \\
	\hline
	$0.8N(0_p, \Sigma_1)+0.2N(1_p, \Sigma_2)$ & 100-20-1 & 0.0817 (0.0213) & 1.2753 (0.4523) \\
	\hline
	$0.8N(0_p, \Sigma_1)+0.2N(6*1_p, \Sigma_2)$ & 100-20-1 & 0.1069 (0.0357) & 1.1668 (0.1839)  \\
	\hline
	$0.8N(0_p, \Sigma_1)+0.2\textnormal{Cauchy}(0.5*1_p)$ & 100-20-1 & 0.0797 (0.0257) & 4.0653 (0.1569) \\
	\hline
   \end{tabular}}
\caption{Numerical experiments for robust mean estimation with unknown covariance trained with $50,000$ samples. The covariance matrices $\Sigma_1$ and $\Sigma_2$ are generated by the same way described in Section \ref{sec:rate}.}\label{tab:unknowcov}
\end{table}

\subsection{Adaptation to Elliptical Distributions}\label{sec:last}

To illustrate the performance of (\ref{eq:JS-GAN-EC}), we conduct a numerical experiment for the estimation of the location parameter $\theta$ with i.i.d. observations $X_1,...,X_n\sim(1-\epsilon)\textnormal{Cauchy}(\theta, I_p) + \epsilon Q$. The density function of $\textnormal{Cauchy}(\theta, I_p)$ is given by $p_{\theta}(x) \propto \left(1+\|x-\theta\|\right)^{-(1+p)/2}$.

Compared with Algorithm \ref{alg:jsgan}, the difference lies in the choice of the generator. We consider the generator
 $G_1(\xi, U)=g_\omega(\xi)U + \theta$, where $g_\omega(\xi)$ is a non-negative neural network parametrized by $\omega$  and some random variable $\xi$. The random vector $U$ is sampled from the uniform distribution on $\{u\in\mathbb{R}^p: \|u\|=1\}$. If the scatter matrix is unknown, we will use the generator $G_2(\xi, U)=g_\omega(\xi)AU + \theta$, with $AA^T$ modeling the scatter matrix.
 
 Table \ref{tab:EC} shows the comparison with other methods. Our method still works well under Cauchy distribution, while the performance of other methods that rely on moment conditions deteriorates in this setting.

\vspace{-2mm}
\begin{table}[H]
\begin{center}
\resizebox{\linewidth}{!}{
      \begin{tabular}{|c|c|c|c|c|}
   	\hline
   	 Contamination $Q$ & JS-GAN ($G_1$) & JS-GAN ($G_2$) & Dimension Halving & Iterative Filtering \\
	\hline\hline
	$\textnormal{Cauchy}(1.5*1_p, I_p)$   & \bf{0.0664 (0.0065)} & 0.0743 (0.0103) & 0.3529 (0.0543) & 0.1244 (0.0114) \\
	\hline
	$\textnormal{Cauchy}(5.0*1_p, I_p)$   & \bf{0.0480 (0.0058)} & 0.0540 (0.0064) & 0.4855 (0.0616)& 0.1687 (0.0310) \\
	\hline
	$\textnormal{Cauchy}(1.5*1_p, 5*I_p)$   & 0.0754 (0.0135) & \bf{0.0742 (0.0111)} & 0.3726 (0.0530) & 0.1220 (0.0112) \\
	\hline
	$\textnormal{Normal}(1.5*1_p, 5*I_p)$    & \bf{0.0702 (0.0064)} & 0.0713 (0.0088) & 0.3915 (0.0232) &  0.1048 (0.0288)) \\
	\hline
   \end{tabular}}
\caption{{Comparison of various methods of robust location estimation under Cauchy distributions. 
Samples are drawn from $(1-\epsilon)\textnormal{Cauchy}(0_p, I_p) + \epsilon Q$ with $\epsilon=0.2, p=50$ and various choices of $Q$. Sample size: 50,000. Discriminator net structure: 50-50-25-1. Generator $g_\omega(\xi)$ structure: 48-48-32-24-12-1 with absolute value activation function in the output layer. 
}}\label{tab:EC}
\end{center}
\end{table}
\vspace{-2mm}

\section{Discussions} \label{sec:disc}


\paragraph{Variational Lower Bounds for Robust Estimation.}

In this paper, we study robust estimation via the technique of generative adversarial nets. We show that the presence of hidden layers are crucial for the estimators trained by JS-GAN to be robust. To better understand the intuition of the results in the paper, we give some further discussion from the perspective of variational lower bounds. In view of (\ref{eq:f-div-variational}), we have
\begin{equation}
\JS(N(\theta,I_p),N(\eta,I_p))\geq \sup_{D\in\mathcal{D}}\left[E_{N(\theta,I_p)}\log D(X)+E_{N(\eta,I_p)}\log(1-D(X))\right]+\log 4, \label{eq:vlb-JS}
\end{equation}
for any discriminator class $\mathcal{D}$. Moreover, according to \citep{nguyen2010estimating,goodfellow2014generative}, the optimal discriminator is achieved at
\begin{equation}
D(X)=\frac{dN(\theta,I_p)}{dN(\theta,I_p)+dN(\eta,I_p)}(X)=\sig\left((\theta-\eta)^TX+\frac{\|\eta\|^2-\|\theta\|^2}{2}\right). \label{eq:optimal-discri}
\end{equation}
Interestingly, (\ref{eq:optimal-discri}) is in the form of logistic regression, and this immediately implies that the variational lower bound (\ref{eq:vlb-JS}) is sharp when we take $\mathcal{D}$ to be the class of logistic regression defined in (\ref{eq:TV-NN1}). Indeed, when there is no contamination or $\epsilon=0$, the sample version of JS-GAN (\ref{eq:JS-GAN-gen}) with the logistic regression discriminator class (\ref{eq:TV-NN1}) leads to the estimator $\wh{\theta}=\frac{1}{n}\sum_{i=1}^nX_i$ according to Proposition \ref{prop:local-JS}, and this is obviously a minimax optimal estimator \citep{lehmann2006theory}.

In contrast, when there is contamination or $\epsilon>0$, the logistic regression discriminator class (\ref{eq:TV-NN1}) does not even lead to a consistent estimator. This is because the population objective function to be minimized is
$$\JS\left((1-\epsilon)N(\theta,I_p)+\epsilon Q,N(\eta,I_p)\right)$$
instead of $\JS(N(\theta,I_p),N(\eta,I_p))$. The variational lower bound with the logistic regression discriminator class (\ref{eq:TV-NN1}) is not sharp anymore because of the presence of the contamination distribution $Q$. In fact, a discriminator class $\mathcal{D}$ that leads to a sharp variational lower bound has to include the function
\begin{equation}
D(X)=\frac{(1-\epsilon)dN(\theta,I_p)+\epsilon dQ}{(1-\epsilon)dN(\theta,I_p)+\epsilon dQ+dN(\eta,I_p)}(X).\label{eq:D-with-Q}
\end{equation}
However, since there is no assumption on the contamination distribution $Q$, the discriminator function (\ref{eq:D-with-Q}) can take an infinite many of forms. As a consequence, a discriminator class $\mathcal{D}$ that includes all possible functions in the form of (\ref{eq:D-with-Q}) will certainly overfit the data, and thus is not practical at all. On the other hand, we show that for the purpose of robust mean estimation, we only need to add an extra hidden layer to the logistic regression discriminator class (\ref{eq:TV-NN1}). The class (\ref{eq:JS-NN2}) of neural nets with one hidden layer does not lead to a sharp variational lower bound, but it is rich enough for the estimator trained by JS-GAN to be robust against any contamination distribution. Moreover, the complexity of the class (\ref{eq:JS-NN2}) is well controlled so that overfitting does not happen and thus the estimator achieves the minimax rate of the problem.

\paragraph{Future Projects.} Besides the topic of robust mean estimation, other important problems include robust covariance matrix estimation, robust high-dimensional regression, robust learning of Gaussian mixture models, and robust classification. It will be interesting to investigate what class of discriminators are suitable for these tasks. Another line of research is motivated from the goal to understand the class of divergence functions that are suitable for robust estimation. In addition to JS-GAN and TV-GAN studied in this paper, we would like to know whether it is possible to train robust estimators using GAN derived from other $f$-divergence functions. A further question is whether it is possible to use GAN derived from integral probability metrics including Wasserstein distance \citep{arjovsky2017wasserstein} and  maximum mean discrepancy \citep{dziugaite2015training,li2015generative,binkowski2018demystifying}. Finally, the landscapes and optimization properties of various GANs under robust estimation settings are topics to be explored.

\section{Proofs}\label{sec:pf}

In this section, we present proofs of all technical results in the paper. We first establish some useful lemmas in Section \ref{sec:aux-lem}, and the the proofs of main theorems will be given in Section \ref{sec:pf-main}.

\subsection{Some Auxiliary Lemmas}\label{sec:aux-lem}

\begin{lemma}\label{lem:EP-NN1}
Given i.i.d. observations $X_1,...,X_n\sim\mathbb{P}$ and the function class $\mathcal{D}$ defined in (\ref{eq:TV-NN1}), we have for any $\delta>0$,
$$\sup_{D\in\mathcal{D}}\left|\frac{1}{n}\sum_{i=1}^nD(X_i)-\mathbb{E}D(X)\right|\leq C\left(\sqrt{\frac{p}{n}}+ \sqrt{\frac{\log(1/\delta)}{n}}\right),$$
with probability at least $1-\delta$ for some universal constant $C>0$.
\end{lemma}
\begin{proof}
Let $f(X_1,...,X_n)=\sup_{D\in\mathcal{D}}\left|\frac{1}{n}\sum_{i=1}^nD(X_i)-\mathbb{E}D(X)\right|$. It is clear that $f(X_1,...,X_n)$ satisfies the bounded difference condition. By McDiarmid's inequality \citep{mcdiarmid1989method}, we have
$$f(X_1,...,X_n)\leq \mathbb{E}f(X_1,...,X_n) + \sqrt{\frac{\log(1/\delta)}{2n}},$$
with probability at least $1-\delta$. Using a standard symmetrization technique \citep{pollard2012convergence}, we obtain the following bound that involves Rademacher complexity,
\begin{equation}
\mathbb{E}f(X_1,...,X_n)\leq 2\mathbb{E}\sup_{D\in\mathcal{D}}\left|\frac{1}{n}\sum_{i=1}^n\epsilon_i D(X_i)\right|,\label{eq:RC-NN1}
\end{equation}
where $\epsilon_1,...,\epsilon_n$ are independent Rademacher random variables. The Rademacher complexity can be bounded by Dudley's integral entropy bound, which gives
$$\mathbb{E}\sup_{D\in\mathcal{D}}\left|\frac{1}{n}\sum_{i=1}^n\epsilon_i D(X_i)\right|\lesssim \mathbb{E}\frac{1}{\sqrt{n}}\int_0^2\sqrt{\log\mathcal{N}(\delta,\mathcal{D},\|\cdot\|_n)}d\delta,$$
where $\mathcal{N}(\delta,\mathcal{D},\|\cdot\|_n)$ is the $\delta$-covering number of $\mathcal{D}$ with respect to the empirical $\ell_2$ distance $\|f-g\|_n=\sqrt{\frac{1}{n}\sum_{i=1}^n(f(X_i)-g(X_i))^2}$. Since the VC-dimension of $\mathcal{D}$ is $O(p)$, we have $\mathcal{N}(\delta,\mathcal{D},\|\cdot\|_n)\lesssim p\left(16e/\delta\right)^{O(p)}$ (see Theorem 2.6.7 of \cite{van1996weak}). This leads to the bound $\frac{1}{\sqrt{n}}\int_0^2\sqrt{\log\mathcal{N}(\delta,\mathcal{D},\|\cdot\|_n)}d\delta\lesssim \sqrt{\frac{p}{n}}$, which gives the desired result. 
\end{proof}

\begin{lemma}\label{lem:JS-NN2}
Given i.i.d. observations $X_1,...,X_n\sim\mathbb{P}$, and the function class $\mathcal{D}$ defined in (\ref{eq:JS-NN2}), we have for any $\delta>0$,
$$\sup_{D\in\mathcal{D}}\left|\frac{1}{n}\sum_{i=1}^n\log D(X_i)-\mathbb{E}\log D(X)\right|\leq C\kappa\left(\sqrt{\frac{p}{n}}+ \sqrt{\frac{\log(1/\delta)}{n}}\right),$$
with probability at least $1-\delta$ for some universal constant $C>0$.
\end{lemma}
\begin{proof}
Let $f(X_1,...,X_n)=\sup_{D\in\mathcal{D}}\left|\frac{1}{n}\sum_{i=1}^n\log D(X_i)-\mathbb{E}\log D(X)\right|$. Since
$$\sup_{D\in\mathcal{D}}\sup_{x}|\log(2D(x))|\leq\kappa,$$
we have
$$\sup_{x_1,...,x_n,x_i'}\left|f(x_1,...,x_n)-f(x_1,...,x_{i-1},x_i',x_{i+1},...,x_n)\right|\leq \frac{2\kappa}{n}.$$
Therefore, by McDiarmid's inequality \citep{mcdiarmid1989method}, we have
\begin{equation}
f(X_1,...,X_n)\leq \mathbb{E}f(X_1,...,X_n) + \kappa\sqrt{\frac{2\log(1/\delta)}{n}},\label{eq:bd-diff-JS-NN2}
\end{equation}
with probability at least $1-\delta$. By the same argument of (\ref{eq:RC-NN1}), it is sufficient to bound the Rademacher complexity $\mathbb{E}\sup_{D\in\mathcal{D}}\left|\frac{1}{n}\sum_{i=1}^n\epsilon_i\log(2D(X_i))\right|$. Since the function $\psi(x)=\log(2\sig(x))$ has Lipschitz constant $1$ and satisfies $\psi(0)=0$, we have
$$\mathbb{E}\sup_{D\in\mathcal{D}}\left|\frac{1}{n}\sum_{i=1}^n\epsilon_i\log(2D(X_i))\right|\leq 2\mathbb{E}\sup_{\sum_{j\geq 1}|w_j|\leq\kappa,u_j\in\mathbb{R}^p,b_j\in\mathbb{R}}\left|\frac{1}{n}\sum_{i=1}^n\epsilon_i\sum_{j\geq 1}w_j\sigma(u_j^TX_i+b_j)\right|,$$
which uses Theorem 12 of \cite{bartlett2002rademacher}. By H\"{o}lder's inequality, we further have
\begin{eqnarray*}
&& \mathbb{E}\sup_{\sum_{j\geq 1}|w_j|\leq\kappa,u_j\in\mathbb{R}^p,b_j\in\mathbb{R}}\left|\frac{1}{n}\sum_{i=1}^n\epsilon_i\sum_{j\geq 1}w_j\sigma(u_j^TX_i+b_j)\right| \\
&\leq& \kappa \mathbb{E}\sup_{j\geq 1}\sup_{u_j\in\mathbb{R}^p,b_j\in\mathbb{R}}\left|\frac{1}{n}\sum_{i=1}^n\epsilon_i\sigma(u_j^TX_i+b_j)\right| \\
&=& \kappa\mathbb{E}\sup_{u\in\mathbb{R}^p,b\in\mathbb{R}}\left|\frac{1}{n}\sum_{i=1}^n\epsilon_i\sigma(u^TX_i+b)\right|.
\end{eqnarray*}
Note that for a monotone function $\sigma:\mathbb{R}\rightarrow [0,1]$, the VC-dimension of the class $\{\sigma(u^Tx+b):u\in\mathbb{R},b\in\mathbb{R}\}$ is $O(p)$. Therefore, by using the same argument of Dudley's integral entropy bound in the proof Lemma \ref{lem:EP-NN1}, we have
$$\mathbb{E}\sup_{u\in\mathbb{R}^p,b\in\mathbb{R}}\left|\frac{1}{n}\sum_{i=1}^n\epsilon_i\sigma(u^TX_i+b)\right|\lesssim \sqrt{\frac{p}{n}},$$
which leads to the desired result.
\end{proof}

\begin{lemma}\label{lem:JS-DNN-sig}
Given i.i.d. observations $X_1,..,X_n\sim N(\theta,I_p)$ and the function class ${\mathcal{F}}_L^H(\kappa,\tau,B)$. Assume $\|\theta\|_{\infty}\leq \sqrt{\log p}$ and set $\tau=\sqrt{p\log p}$. We have
for any $\delta>0$,
$$\sup_{D\in{\mathcal{F}}_L^H(\kappa,\tau,B)}\left|\frac{1}{n}\sum_{i=1}^n\log D(X_i)-\mathbb{E}\log D(X)\right|\leq C\kappa\left((2B)^{L-1}\sqrt{\frac{p\log p}{n}}+\sqrt{\frac{\log(1/\delta)}{n}}\right),$$
with probability at least $1-\delta$ for some universal constants $C>0$.
\end{lemma}
\begin{proof}
Write $f(X_1,...,X_n)=\sup_{D\in{\mathcal{F}}_L^H(\kappa,\tau,B)}\left|\frac{1}{n}\sum_{i=1}^n\log D(X_i)-\mathbb{E}\log D(X)\right|$. Then, the inequality (\ref{eq:bd-diff-JS-NN2}) holds with probability at least $1-\delta$. It is sufficient to analyze the Rademacher complexity. Using the fact that the function $\log(2\sig(x))$ is Lipschitz and H\"{o}lder's inequality, we have
\begin{eqnarray*}
&& \mathbb{E}\sup_{D \in{\mathcal{F}}_L^H(\kappa,\tau,B)}\left|\frac{1}{n}\sum_{i=1}^n\epsilon_i\log(2D(X_i))\right| \\
&\leq& 2\mathbb{E}\sup_{\|w\|_1\leq\kappa, \|u_{j*}\|^2\leq 2,|b_j|\leq\tau, g_{jh}\in\mathcal{G}_{L-1}^H(B)}\left|\frac{1}{n}\sum_{i=1}^n\epsilon_i\sum_{j\geq 1}w_j\sig\left(\sum_{h=1}^{2p}u_{jh}g_{jh}(X_i)+b_j\right)\right| \\
&\leq& 2\kappa\mathbb{E}\sup_{\|u\|^2\leq 2,|b|\leq\tau,g_h\in\mathcal{G}_{L-1}^H(B)}\left|\frac{1}{n}\sum_{i=1}^n\epsilon_i\sig\left(\sum_{h=1}^{2p}u_hg_h(X_i)+b\right)\right| \\
&\leq& 4\kappa\mathbb{E}\sup_{\|u\|^2\leq 2,|b|\leq\tau,g_h\in\mathcal{G}_{L-1}^H(B)}\left|\frac{1}{n}\sum_{i=1}^n\epsilon_i\left(\sum_{h=1}^{2p}u_hg_h(X_i)+b\right)\right| \\
&\leq& 8\sqrt{p}\kappa\mathbb{E}\sup_{g\in\mathcal{G}_{L-1}^H(B)}\left|\frac{1}{n}\sum_{i=1}^n\epsilon_ig(X_i)\right| + 4\kappa\tau\mathbb{E}\left|\frac{1}{n}\sum_{i=1}^n\epsilon_i\right|.
\end{eqnarray*}
Now we use the notation $Z_i=X_i-\theta\sim N(0,I_p)$ for $i=1,...,n$. We bound $\mathbb{E}\sup_{g\in\mathcal{G}_{L-1}^H(B)}\left|\frac{1}{n}\sum_{i=1}^n\epsilon_ig(Z_i+\theta)\right|$ by induction.
Since
\begin{eqnarray*}
&& \mathbb{E}\left(\sup_{g\in\mathcal{G}_{1}^H(B)}\frac{1}{n}\sum_{i=1}^n\epsilon_ig(Z_i+\theta)\right) \\
&\leq& \mathbb{E}\left(\sup_{\|v\|_1\leq B}\frac{1}{n}\sum_{i=1}^n\epsilon_iv^T(Z_i+\theta)\right) \\
&\leq& B\left(\mathbb{E}\left|\frac{1}{n}\sum_{i=1}^n\epsilon_iZ_i\right|_{\infty} + \|\theta\|_{\infty}\mathbb{E}\left|\frac{1}{n}\sum_{i=1}^n\epsilon_i\right|\right) \\
&\leq& CB\frac{\sqrt{\log p}+\|\theta\|_{\infty}}{\sqrt{n}},
\end{eqnarray*}
and
\begin{eqnarray*}
&& \mathbb{E}\left(\sup_{g\in\mathcal{G}_{l+1}^H(B)}\frac{1}{n}\sum_{i=1}^n\epsilon_ig(Z_i+\theta)\right) \\
&\leq& \mathbb{E}\left(\sup_{\|v\|_1\leq B,g_h\in\mathcal{G}_l^H(B)}\frac{1}{n}\sum_{i=1}^n\epsilon_i \sum_{h=1}^Hv_hg_h(Z_i+\theta)\right) \\
&\leq& B\mathbb{E}\left(\sup_{g\in\mathcal{G}_l^H(B)}\left|\frac{1}{n}\sum_{i=1}^n\epsilon_ig(Z_i+\theta)\right|\right) \\
&\leq& 2B\mathbb{E}\left(\sup_{g\in\mathcal{G}_{l}^H(B)}\frac{1}{n}\sum_{i=1}^n\epsilon_ig(Z_i+\theta)\right),
\end{eqnarray*}
we have
$$\mathbb{E}\left(\sup_{g\in\mathcal{G}_{L-1}^H(B)}\frac{1}{n}\sum_{i=1}^n\epsilon_ig(Z_i+\theta)\right)\leq C(2B)^{L-1}\frac{\sqrt{\log p}+\|\theta\|_{\infty}}{\sqrt{n}}.$$
Combining the above inequalities, we get
$$\mathbb{E}\left(\sup_{D\in\mathcal{F}_L^H(\kappa,\tau,B)}\frac{1}{n}\sum_{i=1}^n\epsilon_i\log D(Z_i+\theta)\right)\leq C\kappa\left(\sqrt{p}(2B)^{L-1}\frac{\sqrt{\log p}+\|\theta\|_{\infty}}{\sqrt{n}}+\frac{\tau}{\sqrt{n}}\right).$$
This leads to the desired result under the conditions on $\tau$ and $\|\theta\|_{\infty}$.
\end{proof}

\subsection{Proofs of Main Theorems} \label{sec:pf-main}

\begin{proof}[Proof of Theorem \ref{thm:TV-NN1}]
We first introduce some notations. Define $F(P,\eta)=\sup_{w,b}F_{w,b}(P,\eta)$, where
$$F_{w,b}(P,\eta)=E_P\sig(w^TX+b)-E_{N(\eta,I_p)}\sig(w^TX+b).$$
With this definition, we have $\wh{\theta}=\arginf_{\eta}F(\mathbb{P}_n,\eta)$, where we use $\mathbb{P}_n$ for the empirical distribution $\frac{1}{n}\sum_{i=1}^n\delta_{X_i}$. We shorthand $N(\eta,I_p)$ by $P_{\eta}$, and then
\begin{eqnarray}
\label{eq:T1-1} F(P_{\theta},\wh{\theta}) &\leq& F((1-\epsilon)P_{\theta}+\epsilon Q,\wh{\theta}) + \epsilon \\
\label{eq:T1-2} &\leq& F(\mathbb{P}_n,\wh{\theta}) + \epsilon + C\left(\sqrt{\frac{p}{n}}+\sqrt{\frac{\log(1/\delta)}{n}}\right) \\
\label{eq:T1-3} &\leq& F(\mathbb{P}_n,\theta) + \epsilon + C\left(\sqrt{\frac{p}{n}}+\sqrt{\frac{\log(1/\delta)}{n}}\right) \\
\label{eq:T1-4} &\leq& F((1-\epsilon)P_{\theta}+\epsilon Q,\theta) + \epsilon + 2C\left(\sqrt{\frac{p}{n}}+\sqrt{\frac{\log(1/\delta)}{n}}\right) \\
\label{eq:T1-5} &\leq& F(P_{\theta},\theta) + 2\epsilon + 2C\left(\sqrt{\frac{p}{n}}+\sqrt{\frac{\log(1/\delta)}{n}}\right) \\
\label{eq:T1-6} &=& 2\epsilon + 2C\left(\sqrt{\frac{p}{n}}+\sqrt{\frac{\log(1/\delta)}{n}}\right).
\end{eqnarray}
With probability at least $1-\delta$, the above inequalities hold. We will explain each inequality. Since
$$F((1-\epsilon)P_{\theta}+\epsilon Q,\eta)=\sup_{w,b}\left[(1-\epsilon)F_{w,b}(P_{\theta},\eta)+\epsilon F_{w,b}(Q,\eta)\right],$$
we have
$$\sup_{\eta}\left|F((1-\epsilon)P_{\theta}+\epsilon Q,\eta)-F(P_{\theta},\eta)\right|\leq\epsilon,$$
which implies (\ref{eq:T1-1}) and (\ref{eq:T1-5}). The inequalities (\ref{eq:T1-2}) and (\ref{eq:T1-4}) are implied by Lemma \ref{lem:EP-NN1} and the fact that
$$\sup_{\eta}\left|F(\mathbb{P}_n,\eta)-F((1-\epsilon)P_{\theta}+\epsilon Q,\eta)\right|\leq \sup_{w,b}\left|\frac{1}{n}\sum_{i=1}^n\sig(w^TX_i+b)-\mathbb{E}\sig(w^TX+b)\right|.$$
The inequality (\ref{eq:T1-3}) is a direct consequence of the definition of $\wh{\theta}$. Finally, it is easy to see that $F(P_{\theta},\theta)=0$, which gives (\ref{eq:T1-6}). In summary, we have derived that with probability at least $1-\delta$,
$$F_{w,b}(P_{\theta},\wh{\theta})\leq 2\epsilon + 2C\left(\sqrt{\frac{p}{n}}+\sqrt{\frac{\log(1/\delta)}{n}}\right),$$
for all $w\in\mathbb{R}^p$ and $b\in\mathbb{R}$. For any $u\in\mathbb{R}^p$ such that $\|u\|=1$, we take $w=u$ and $b=-u^T\theta$, and we have
$$f(0)-f(u^T(\theta-\wh{\theta}))\leq 2\epsilon + 2C\left(\sqrt{\frac{p}{n}}+\sqrt{\frac{\log(1/\delta)}{n}}\right),$$
where $f(t)=\int \frac{1}{1+e^{z+t}}\phi(z)dz$, with $\phi(\cdot)$ being the probability density function of $N(0,1)$. It is not hard to see that as long as $|f(t)-f(0)|\leq c$ for some sufficiently small constant $c>0$, then $|f(t)-f(0)|\geq c'|t|$ for some constant $c'>0$. This implies
\begin{eqnarray*}
\|\wh{\theta}-\theta\| &=& \sup_{\|u\|=1}|u^T(\wh{\theta}-\theta)| \\
&\leq& \frac{1}{c'}\sup_{\|u\|=1}\left|f(0)-f(u^T(\theta-\wh{\theta}))\right| \\
&\lesssim& \epsilon + \sqrt{\frac{p}{n}}+\sqrt{\frac{\log(1/\delta)}{n}},
\end{eqnarray*}
with probability at least $1-\delta$. The proof is complete.
\end{proof}

\begin{proof}[Proof of Theorem \ref{thm:JS-NN2}]
We continue to use $P_{\eta}$ to denote $N(\eta,I_p)$. Define
$$F(P,\eta)=\sup_{\|w\|_1\leq\kappa,u,b}F_{w,u,b}(P,\eta),$$
where
$$F_{w,u,b}(P,\eta)=E_P\log D(X)+E_{N(\eta,I_p)}\log\left(1-D(X)\right)+\log 4,$$
with $D(x)=\sig\left(\sum_{j\geq 1}w_j\sigma(u_j^Tx+b_j)\right)$. Then,
\begin{eqnarray}
\label{eq:T2-1} F(P_{\theta},\wh{\theta}) &\leq& F((1-\epsilon)P_{\theta}+\epsilon Q,\wh{\theta}) + 2\kappa\epsilon \\
\label{eq:T2-2} &\leq& F(\mathbb{P}_n,\wh{\theta}) + 2\kappa\epsilon + C\kappa\left(\sqrt{\frac{p}{n}}+\sqrt{\frac{\log(1/\delta)}{n}}\right) \\
\label{eq:T2-3} &\leq& F(\mathbb{P}_n,\theta) + 2\kappa\epsilon + C\kappa\left(\sqrt{\frac{p}{n}}+\sqrt{\frac{\log(1/\delta)}{n}}\right) \\
\label{eq:T2-4} &\leq& F((1-\epsilon)P_{\theta}+\epsilon Q,\theta) + 2\kappa\epsilon + 2C\kappa\left(\sqrt{\frac{p}{n}}+\sqrt{\frac{\log(1/\delta)}{n}}\right) \\
\label{eq:T2-5} &\leq& F(P_{\theta},\theta) + 4\kappa\epsilon + 2C\kappa\left(\sqrt{\frac{p}{n}}+\sqrt{\frac{\log(1/\delta)}{n}}\right) \\
\nonumber &=& 4\kappa\epsilon + 2C\kappa\left(\sqrt{\frac{p}{n}}+\sqrt{\frac{\log(1/\delta)}{n}}\right).
\end{eqnarray}
The inequalities (\ref{eq:T2-1})-(\ref{eq:T2-5}) follow similar arguments for (\ref{eq:T1-1})-(\ref{eq:T1-5}). To be specific, (\ref{eq:T2-2}) and (\ref{eq:T2-4}) are implied by Lemma \ref{lem:JS-NN2}, and (\ref{eq:T2-3}) is a direct consequence of the definition of $\wh{\theta}$. To see (\ref{eq:T2-1}) and (\ref{eq:T2-5}), note that for any $w$ such that $\|w\|_1\leq \kappa$, we have
$$|\log (2D(X))|\leq \left|\sum_{j\geq 1}w_j\sigma(u_j^TX+b_j)\right|\leq \kappa.$$
A similar argument gives the same bound for $|\log(2(1-D(X)))|$. This leads to
$$\sup_{\eta}\left|F((1-\epsilon)P_{\theta}+\epsilon Q,\eta)-F(P_{\theta},\eta)\right|\leq 2\kappa\epsilon,$$
which further implies (\ref{eq:T2-1}) and (\ref{eq:T2-5}). To summarize, we have derived that with probability at least $1-\delta$,
$$F_{w,u,b}(P_{\theta},\wh{\theta})\leq 4\kappa\epsilon + 2C\kappa\left(\sqrt{\frac{p}{n}}+\sqrt{\frac{\log(1/\delta)}{n}}\right),$$
for all $\|w\|_1\leq \kappa$, $\|u_j\|\leq 1$ and $b_j$. Take $w_1=\kappa$, $w_j=0$ for all $j>1$, $u_1=u$ for some unit vector $u$ and $b_1=-u^T\theta$, and we get
\begin{equation}
f_{u^T(\wh{\theta}-\theta)}(\kappa)\leq 4\kappa\epsilon + 2C\kappa\left(\sqrt{\frac{p}{n}}+\sqrt{\frac{\log(1/\delta)}{n}}\right), \label{eq:not-c}
\end{equation}
where
\begin{equation}
f_{\delta}(t)=\mathbb{E}\log\frac{2}{1+e^{-t\sigma(Z)}}+\mathbb{E}\log\frac{2}{1+e^{t\sigma(Z+\delta)}},\label{eq:def-f-delta-t}
\end{equation}
with $Z\sim N(0,1)$.
Direct calculations give
\begin{eqnarray}
\nonumber f_{\delta}'(t) &=& \mathbb{E}\frac{e^{-t\sigma(Z)}}{1+e^{-t\sigma(Z)}}\sigma(Z) - \mathbb{E}\frac{e^{t\sigma(Z+\delta)}}{1+e^{t\sigma(Z+\delta)}}\sigma(Z+\delta), \\
\label{eq:f''} f_{\delta}''(t) &=& -\mathbb{E}\sigma(Z)^2\frac{e^{-t\sigma(Z)}}{(1+e^{-t\sigma(Z)})^2} - \mathbb{E}\sigma(Z+\delta)^2\frac{e^{t\sigma(Z+\delta)}}{(1+e^{t\sigma(Z+\delta)})^2}.
\end{eqnarray}
Therefore, $f_{\delta}(0)=0$, $f_{\delta}'(0)=\frac{1}{2}\left(\mathbb{E}\sigma(Z)-\mathbb{E}\sigma(Z+\delta)\right)$, and $f''_{\delta}(t)\geq-\frac{1}{2}$. By the inequality
$$f_{\delta}(\kappa)\geq f_{\delta}(0) + \kappa f_{\delta}'(0) - \frac{1}{4}\kappa^2,$$
we have $\kappa f_{\delta}'(0)\leq f_{\delta}(\kappa)+\kappa^2/4$. In view of (\ref{eq:not-c}), we have
\begin{eqnarray*}
&& \frac{\kappa}{2}\left(\int\sigma(z)\phi(z)dz-\int\sigma(z+u^T(\wh{\theta}-\theta))\phi(z)dz\right) \\
&\leq& 4\kappa\epsilon + 2C\kappa\left(\sqrt{\frac{p}{n}}+\sqrt{\frac{\log(1/\delta)}{n}}\right) + \frac{\kappa^2}{4}.
\end{eqnarray*}
It is easy to see that for the choices of $\sigma(\cdot)$, $\int\sigma(z)\phi(z)dz-\int\sigma(z+t)\phi(z)dz$ is locally linear with respect to $t$. This implies that
$$\kappa\|\wh{\theta}-\theta\|=\kappa\sup_{\|u\|=1}u^T(\wh{\theta}-\theta)\lesssim \kappa\left(\epsilon+\sqrt{\frac{p}{n}}+\sqrt{\frac{\log(1/\delta)}{n}}\right)+\kappa^2.$$
Therefore, with a $\kappa\lesssim \sqrt{\frac{p}{n}}+\epsilon$, the proof is complete.
\end{proof}

\begin{proof}[Proof of Theorem \ref{thm:JS-DNN}]
We continue to use $P_{\eta}$ to denote $N(\eta,I_p)$. Define
$$F(P,\eta)=\sup_{D\in\mathcal{F}_L^H(\kappa,\tau,B)}F_{D}(P,\eta),$$
with
$$
F_{D}(P,\eta)=E_P\log D(X)+E_{N(\eta,I_p)}\log(1-D(X))+\log 4.\label{eq:def-F-D-eta}
$$
Follow the same argument in the proof of Theorem \ref{thm:JS-NN2}, use Lemma \ref{lem:JS-DNN-sig}, and we have
$$F_D(P_{\theta},\wh{\theta})\leq C\kappa\left(\epsilon+(2B)^{L-1}\sqrt{\frac{p\log p}{n}}+\sqrt{\frac{\log(1/\delta)}{n}}\right),$$
uniformly over $D\in{\mathcal{F}}_L^H(\kappa,\tau,B)$ with probability at least $1-\delta$. Choose $w_1=\kappa$ and $w_j=0$ for all $w_j>1$. For any unit vector $\wt{u}\in\mathbb{R}^p$, take $u_{1h}=-u_{1(h+p)}=\wt{u}_h$ for $h=1,...,p$ and $b_1=-\wt{u}^T\theta$. For $h=1,...,p$, set $g_{1h}(x)=\max(x_h,0)$. For $h=p+1,...,2p$, set $g_{1h}(x)=\max(-x_{h-p},0)$. It is obvious that such $u$ and $b$ satisfy $\sum_hu_{1h}^2\leq 2$ and $|b_1|\leq \|\theta\|\leq \sqrt{p}\|\theta\|_{\infty}\leq\sqrt{p\log p}$. We need to show both the functions $\max(x,0)$ and $\max(-x,0)$ are elements of $\mathcal{G}_{L-1}^H(B)$. This can be proved by induction. It is obvious that $\max(x_h,0), \max(-x_h,0)\in\mathcal{G}_1^H(B)$ for any $h=1,...,p$. Suppose we have $\max(x_h,0), \max(-x_h,0)\in\mathcal{G}_l^H(B)$ for any $h=1,...,p$. Then,
\begin{eqnarray*}
\max\left(\max(x_h,0)-\max(-x_h,0),0\right) &=& \max(x_h,0), \\
\max\left(\max(-x_h,0)-\max(x_h,0),0\right) &=& \max(-x_h,0).
\end{eqnarray*}
Therefore, $\max(x_h,0), \max(-x_h,0)\in\mathcal{G}_{l+1}^H(B)$ as long as $B\geq 2$. Hence, the above construction satisfies $D(x)=\sig(\kappa\sig(\wt{u}^T(x-\theta)))\in\mathcal{F}_L^H(\kappa,\tau,B)$, and we have
\begin{equation}
f_{u^T(\wh{\theta}-\theta)}(\kappa)\leq C\kappa\left(\epsilon+(2B)^{L-1}\sqrt{\frac{p\log p}{n}}+\sqrt{\frac{\log(1/\delta)}{n}}\right), \label{eq:not-c-again}
\end{equation}
where the definition of $f_{\delta}(t)$ is given by (\ref{eq:def-f-delta-t}) with $Z\sim N(0,1)$ and $\sigma(\cdot)$ is taken as $\sig(\cdot)$. Apply the a similar  in the proof of Theorem \ref{thm:JS-NN2}, we obtain
the desired result.
\end{proof}

\begin{proof}[Proof of Theorem \ref{thm:elliptical}]
We use $P_{\theta,\Sigma,h}$ to denote the elliptical distribution $EC(\theta,\Sigma,h)$. 
Define
$$F(P,(\eta,\Gamma,g))=\sup_{\|w\|_1\leq\kappa,u,b}F_{w,u,b}(P,(\eta,\Gamma,g)),$$
where
$$F_{w,u,b}(P,(\eta,\Gamma,g))=E_P\log D(X)+E_{EC(\eta,\Gamma,g)}\log\left(1-D(X)\right)+\log 4,$$
with $D(x)=\sig\left(\sum_{j\geq 1}w_j\sigma(u_j^Tx+b_j)\right)$.
The same argument in Theorem \ref{thm:JS-NN2} leads to the fact that with probability at least $1-\delta$,
$$F_{w,u,b}(P_{\theta,\Sigma,h},(\wh{\theta},\wh{\Sigma},\wh{h}))\leq 4\kappa\epsilon + 2C\kappa\left(\sqrt{\frac{p}{n}}+\sqrt{\frac{\log(1/\delta)}{n}}\right),$$
for all $\|w\|_1\leq \kappa$, $\|u_j\|\leq 1$ and $b_j$. Take $w_1=\kappa$, $w_j=0$ for all $j>1$, $u_1=u/\sqrt{u^T\wh{\Sigma}u}$ for some unit vector $u$ and $b_1=-u^T\theta/\sqrt{u^T\wh{\Sigma}u}$, and we get
$$f_{\frac{u^T(\wh{\theta}-\theta)}{\sqrt{u^T\wh{\Sigma}u}}}(\kappa)\leq 4\kappa\epsilon + 2C\kappa\left(\sqrt{\frac{p}{n}}+\sqrt{\frac{\log(1/\delta)}{n}}\right),$$
where
$$f_{\delta}(t)=\int \log\left(\frac{2}{1+e^{-t\sigma(\Delta s)}}\right)h(s)ds+\int\log\left(\frac{2}{1+e^{t\sigma(\delta+s)}}\right)\wh{h}(s)ds,$$
where $\delta=\frac{u^T(\wh{\theta}-\theta)}{\sqrt{u^T\wh{\Sigma}u}}$ and $\Delta=\frac{\sqrt{u^T\Sigma u}}{\sqrt{u^T\wh{\Sigma}u}}$. A similar argument to the proof of Theorem \ref{thm:JS-NN2} gives
\begin{eqnarray*}
&& \frac{\kappa}{2}\left(\int\sigma(\Delta s)h(s)ds-\int \sigma(\delta+s)\wh{h}(s)ds\right) \\
&\leq& 4\kappa\epsilon + 2C\kappa\left(\sqrt{\frac{p}{n}}+\sqrt{\frac{\log(1/\delta)}{n}}\right) + \frac{\kappa^2}{4}.
\end{eqnarray*}
Since
$$\int\sigma(\Delta s)h(s)ds=\frac{1}{2}=\int\sigma(s)\wh{h}(s)ds,$$
the above bound is equivalent to
$$\frac{\kappa}{2}\left(H(0)-H(\delta)\right)\leq 4\kappa\epsilon + 2C\kappa\left(\sqrt{\frac{p}{n}}+\sqrt{\frac{\log(1/\delta)}{n}}\right) + \frac{\kappa^2}{4},$$
where $H(\delta)=\int \sigma(\delta+s)\wh{h}(s)ds$. The above bound also holds for $\frac{\kappa}{2}(H(\delta)-H(0))$ by a symmetric argument, and therefore the same bound holds for $\frac{\kappa}{2}|H(\delta)-H(0)|$.
Since $H'(0)=\int\sigma(s)(1-\sigma(s))\wh{h}(s)ds=1$, $H(\delta)$ is locally linear at $\delta=0$, which leads to a desired bound for $\delta=\frac{u^T(\wh{\theta}-\theta)}{\sqrt{u^T\wh{\Sigma}u}}$. Finally, since $u^T\wh{\Sigma}u\leq M$, we get the bound for $u^T(\wh{\theta}-\theta)$. The proof is complete by taking supreme of $u$ over the class of all unit vectors.
\end{proof}

\section*{Acknowledgement}

The research of Chao Gao was supported in part by NSF grant DMS-1712957 and NSF Career Award DMS-1847590.
The research of Yuan Yao was supported in part by Hong Kong Research Grant Council (HKRGC) grant 16303817, National Basic Research Program of China (No. 2015CB85600), National Natural Science Foundation of China (No. 61370004, 11421110001), as well as awards from Tencent AI Lab, Si Family Foundation, Baidu Big Data Institute, and Microsoft Research-Asia.

\bibliographystyle{plainnat}
\bibliography{Robust}


\end{document}